\documentclass[journal]{IEEEtran}

\usepackage{amsmath}
\usepackage{amsthm}
\usepackage{amssymb}
\usepackage{multirow}
\usepackage{url} 
\usepackage[ruled]{algorithm2e}
\usepackage[pagebackref=False,breaklinks=true,colorlinks,bookmarks=false]{hyperref}

\newtheorem{assumption}{Assumption}
\newtheorem{theorem}{Theorem}

\newtheorem{remark}{Remark}
\newtheorem{lemma}{Lemma}

\usepackage[dvipsnames]{xcolor}

\usepackage{colortbl}
\usepackage{pifont}
\newcommand{\cmark}{\ding{51}}%
\newcommand{\xmark}{\ding{55}}%
\usepackage{arydshln}
\usepackage[position=below]{subfig}
\usepackage{graphicx}
\usepackage[T1]{fontenc} 

\begin{document}

\title{Compressing Features for Learning \\ with Noisy Labels}

\author{Yingyi Chen,
        Shell Xu Hu,
        Xi Shen,
        Chunrong Ai,
        and~Johan A.K.~Suykens
\thanks{This work is jointly supported by ERC Advanced Grant E-DUALITY (787960), KU Leuven Grant C14/18/068, Grant FWO GOA4917N, Grant from the Flemish Government (AI Research Program), EU H2020 ICT-48 Network TAILOR, Leuven.AI Institute.
\textit{(Corresponding authors: Yingyi Chen; Xi Shen.)}}
\thanks{Yingyi Chen is with STADIUS, ESAT, KU Leuven, 3001 Leuven, Belgium
(e-mail: yingyi.chen@esat.kuleuven.be).}
\thanks{Shell Xu Hu is with Samsung AI Center, Cambridge, England 
(e-mail: shell.hu@samsung.com).}
\thanks{Xi Shen is with Tencent AI lab, Shenzhen, China 
(e-mail: xi.shen@enpc.fr).}
\thanks{Chunrong Ai is with School of Management and Economics, Chinese University of Hong Kong, Shenzhen, China 
(e-mail: chunrongai@cuhk.edu.cn).}
\thanks{Johan A.K.~Suykens is with STADIUS, ESAT, KU Leuven, 3001 Leuven, Belgium 
(e-mail: johan.suykens@esat.kuleuven.be)}
}

\markboth{}%
{Shell \MakeLowercase{\textit{et al.}}: Bare Demo of IEEEtran.cls for IEEE Journals}

\maketitle

\begin{abstract}
    Supervised learning can be viewed as distilling relevant information from input data into feature representations.
    This process becomes difficult when supervision is noisy as the distilled information might not be relevant.
    In fact, recent research~\cite{zhang2017understanding} shows that networks can easily overfit all labels including those that are corrupted, and hence can hardly generalize to clean datasets.
    In this paper, we focus on the problem of learning with noisy labels and introduce compression inductive bias to network architectures to alleviate this over-fitting problem.
    More precisely, we revisit one classical  regularization named Dropout~\cite{srivastava2014dropout} 
    and its variant Nested Dropout~\cite{rippel2014learning}.
    Dropout can serve as a compression constraint for its feature dropping mechanism,
    while Nested Dropout further learns ordered feature representations w.r.t.~feature importance.
    Moreover, the trained models with compression regularization are further combined with Co-teaching~\cite{han2018co} for performance boost.

    Theoretically, we conduct bias-variance decomposition of the objective function under compression regularization. 
    We analyze it for both single model and Co-teaching.
    This decomposition provides three insights:
    \textit{(i)} it shows that over-fitting is indeed an issue in learning with noisy labels; 
    \textit{(ii)} through an information bottleneck formulation, it explains why the proposed feature compression helps in combating label noise;
    \textit{(iii)} it gives explanations on the performance boost brought by incorporating compression regularization into Co-teaching.
    Experiments show that our simple approach can have comparable or even better performance than the state-of-the-art methods on benchmarks with real-world label noise including Clothing1M \cite{xiao2015learning} and ANIMAL-10N \cite{song2019selfie}.
    Our implementation is available at \href{https://yingyichen-cyy.github.io/CompressFeatNoisyLabels/}{https://yingyichen-cyy.github.io/CompressFeatNoisyLabels/}.
\end{abstract}


\section{Introduction} \label{sec::introduction}
\IEEEPARstart{T}{he} success of deep learning depends on the availability of massive and carefully labeled data. 
However, there is often no guarantee that all annotations are perfect, especially when the amount of data is huge and annotations are required to be fine such as optical flow and segmentation.
In contrast, with the rapid development of the Internet, there are multiple ways to have inexpensive and convenient access to large but defective data, including querying commercial search engines \cite{li2017webvision}, downloading images from social media \cite{mahajan2018exploring}, and various web crawling strategies \cite{olston2010web}.
Correspondingly, persistent efforts have been paid in literature to learn with imperfect data, among which learning with noisy labels has always been attached great significance.

The problem of learning with noisy labels dates back to \cite{angluin1988learning, quinlan1986induction}.
The mainstream methods include 
\textit{(i)}
training on reweighted samples \cite{han2018co, jiang2018mentornet, malach2017decoupling, yu2019does, wei2020combating} where samples possibly clean are assigned larger weights than those possibly corrupted;
\textit{(ii)} 
employing robust loss functions to resist noise \cite{patrini2017making, natarajan2013learning, reed2014training, goldberger2017training};
\textit{(iii)}
conducting label correction \cite{tanaka2018joint, yi2019probabilistic, zhang2021learning} where original labels are often substituted by the possible clean predictions;
\textit{(iv)}
semi-supervised learning methods
\cite{ding2018semi,kong2019recycling,li2020dividemix}
where samples are first identified as clean or corrupted, and then networks are trained in a semi-supervised manner with only the clean labels used.
Moreover, label noise itself also plays an important role in understanding the \textit{generalization puzzle} of deep learning.
Empirical experiments in \cite{zhang2017understanding} show that deep neural networks (DNNs), such as AlexNet \cite{krizhevsky2012imagenet}, can achieve almost zero training errors on randomly labelled datasets.
This analysis demonstrates that the capacities of DNNs are often high enough to memorize the entire noisy training information.
Since over-fitting is mainly due to the model capacities, an alternative way to address the problem of training with noisy labels is to introduce explicit compression inductive bias to the model architecture, which is the main focus of this work.

In this paper, we propose to combat this over-fitting problem by introducing compression inductive bias to networks.
More precisely, rather than relying on the prediction of deterministic DNNs, we introduce feature compression
to the hidden features in networks via \textit{Dropout} \cite{srivastava2014dropout} and its variant \textit{Nested Dropout} \cite{rippel2014learning}.
Dropout can be served as a compression constraint for its feature dropping mechanism, while Nested Dropout further learns ordered feature representations w.r.t. feature importance.
Leveraging Nested Dropout, we can not only constrain the model capacity, but also filter out the irrelevance while preserves the relevance w.r.t.~the learning task.
Moreover, compared to Dropout, the information sorting property of Nested Dropout is particularly useful for conducting signal-to-noise separation in the feature level.
Note that we may also consider other compression strategies such as principal component analysis (PCA) and kernel principal component analysis (kernel PCA) \cite{suykens2002least}, but Dropout/Nested Dropout is a plug-and-play component to networks, thus bringing much convenience to the implementation.

In addition to Dropout/Nested Dropout's bringing feature-level compression to networks, we find that they are suitable for incorporating into Co-teaching \cite{han2018co} which is a strong method for learning with noisy labels, for performance boost.
Specifically, Co-teaching trains two networks simultaneously where networks update themselves based on the small-loss mini-batch samples selected by their peer.
Intuitively, this sample selection mechanism discards samples with possibly wrong labels, and preserves those that are possibly clean.
We will show in this paper that the sample selection during the cross-update process together with compression techniques
will further prevent networks from over-fitting the noisy labels.
On account that good performance of Co-teaching requires the two base networks to be reliable enough, we propose our two-stage method:

\begin{itemize}
    \item Train two \textit{Dropout} / \textit{Nested Dropout} networks separately until convergence;
    \item Fine-tune these two networks with \textit{Co-teaching}.
\end{itemize}
Note that Dropout/Nested Dropout is maintained in the second stage for fine-tuning.
The efficacy of our two-stage compression approach is validated on benchmark real-world datasets by achieving comparable or even better performance than the state-of-the-art approaches.
For example, on Clothing1M~\cite{xiao2015learning}, our method obtains 75.0\% in accuracy, which achieves comparable performance to DivideMix~\cite{li2020dividemix} and ELR+~\cite{liu2020early}. 
On ANIMAL-10N~\cite{song2019selfie}, 
we achieve 84.5\% in accuracy while the state-of-the-art method PLC~\cite{zhang2021learning} is 83.4\%.

Beyond the empirical contributions, we provide theoretical explanations on why compression can combat label noise.
In particular, we conduct bias-variance decomposition of the objective function where Dropout/Nested Dropout is formulated into latent variable model.
This decomposition provides three insights:
\textit{(i)} it shows that over-fitting is indeed an issue in learning with noisy labels.
The bias term determines how close the model fits noisy labels, while the variance term promotes a consensus among individual models in latent variable model.
Deterministic DNNs have zero variance term and thus focus on minimize the bias term during training, leading to over-fitting on noisy labels.
\textit{(ii)} Through an information bottleneck formulation, it explains why the proposed feature compression helps in combating label noise.
Dropout/Nested Dropout can serve as compression constraints since they can be formulated as optimizing an information bottleneck.
These compression constraints bring non-zero variance term and thus reduce the impact of the bias term.
\textit{(iii)} It explains the performance boost brought by incorporating compression regularization into Co-teaching.
The cross-update strategy of Co-teaching together with the compression constraints bring larger variance term to further diminish the influence of the bias term, leading to even less over-fitting on the noisy labels.

This paper is based on our previous work \cite{chen2021boosting} which mainly focuses on the empirical results.
We enrich it with theoretical understanding of our method, the learning with noisy labels problem itself, and more detailed numerical assessments.
This paper is structured as follows: 
Section \ref{sec::relatedWork} summarizes the related works in learning with noisy labels.
Section \ref{sec::method} presents our algorithm.
Section \ref{sec::theory} provides a theoretical understanding of our method.
Section \ref{sec::experiments} shows illustrative toy example and experiments on benchmark real-world datasets. 
Finally, we conclude this paper in Section \ref{sec::conclusion}. 
Implementation is available at \href{https://yingyichen-cyy.github.io/CompressFeatNoisyLabels/}{https://yingyichen-cyy.github.io/CompressFeatNoisyLabels/}.

\section{Related Works} \label{sec::relatedWork}
In this section, we briefly review the existing works related to learning with label noise.
Extensive literature reviews can be found in \cite{song2020learning,han2020survey,cordeiro2020survey}.

\paragraph{Over-fitting prevention}
The idea of preventing networks from over-fitting for better generalization has been considered in \cite{arpit2017closer,ma2018dimensionality}.
In particular, \cite{arpit2017closer} proposes that appropriately tuned explicit regularization prevents DNNs from over-fitting noisy datasets while maintains generalization on
clean data, and
\cite{ma2018dimensionality} proposes to understand the generalization of DNNs by investigating the dimensionality of the deep representation subspace of training samples.
C2D~\cite{zheltonozhskii2022contrast} uses self-supervised pre-training to learn more meaningful information before over-fitting to noise.
ELR/ELR+~\cite{liu2020early} proposes an early-learning regularization to resist over-fitting, while AugDesc~\cite{nishi2021augmentation} achieves this by employing different augmentation strategies.
Although starting from the point of preventing networks from over-fitting, our method is different from their works mainly in that 
\textit{(i)} we theoretically verify that over-fitting is indeed an issue by conducting bias-variance decomposition while \cite{arpit2017closer, ma2018dimensionality, zheltonozhskii2022contrast,nishi2021augmentation} are more from an experimental perspective.
\textit{(ii)} we inject extrinsic compression to filter out noisy information, while \cite{ma2018dimensionality} identifies network's intrinsic compression point and adapts the corresponding loss.

\paragraph{Samples reweighting}
Samples reweighting scheme learns to assign small weights to those samples supposed to be corrupted.
ActiveBias~\cite{chang2017active} reweights samples based on the variance of prediction probabilities and the closeness between the prediction probabilities and the decision threshold.
MentorNet \cite{jiang2018mentornet} trains its student network based on the clean samples selected by its teacher network.
Co-teaching \cite{han2018co} cross-updates its two base models 
based on the samll-loss samples selected by their peers.
Decoulping \cite{malach2017decoupling} updates the networks based on samples where the predictions of the two predictors are different, that is, the ``disagreement" strategy.
As for Co-teaching+ \cite{yu2019does}, it combines Co-teaching with the ``disagreement" strategy to further improve the performance.
Different from above where networks are based on ``disagreement", JoCoR \cite{wei2020combating} trains two networks as a whole by a joint loss following the ``agreement" strategy, and select the small-loss examples to update themselves.
This ``agreement" strategy shows improvement over the previous methods.
Note that we still base our method on Co-teaching since it is easier and also effective for both implementation and analysis.
\cite{pleiss2020identifying} proposes a statistic, namely AUM, which differentiates clean samples from mislabeled samples by exploiting their training dynamics.
The mislabeled ones are discarded during training. 
\cite{pleiss2020identifying} is categorized here since discarding samples is equivalent to assigning zero weights to them.

\begin{figure}[t]
	\begin{minipage}[t]{0.45\textwidth}  
		\centering  
		\includegraphics[width=\textwidth]{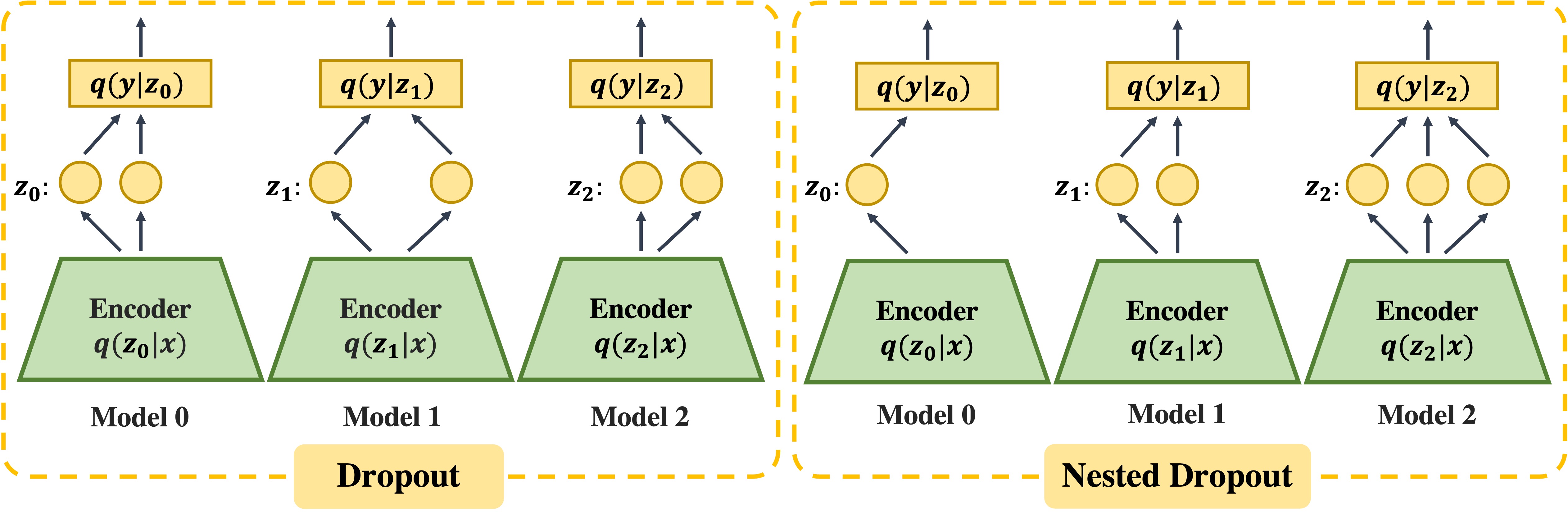}
	\end{minipage}  
	\centering  
	\caption{Both Dropout and Nested Dropout are ways to induct ensemble of models, not merely regularizations. 
	To be specific, network with them applied can be regarded as a latent variable model $q(y|x) := \int_{\mathcal{Z}} q(y|z) q(z|x) \, dz$, where $q(z|x)$ is the encoder, and $q(y|z)$ is the decoder as in \eqref{eq::lvm}.
	``Models 0-2" are models corresponding to different trials of $z$.
	For instance, ``Model 1" of Dropout is the model corresponding to trial $z_1$ where the second neuron is masked out. 
	In this case, the entire Dropout model can be viewed as an ensemble of different trial models since the integral is over all possible $z\in\mathcal{Z}$. 
	Similar explanations also apply to Nested Dropout.}
	\label{fig::compress}
	\vspace{-5mm}
\end{figure}

\paragraph{Robust loss function}
Robust losses have been applied to achieve noise-tolerant classifications including ramp loss \cite{brooks2011support}, unhinged loss \cite{van2015learning}, mean absolute error \cite{suykens2002least, ghosh2017robust, suykens1999least}. 
However, the fact that DNNs can learn arbitrary labels may dampen the effectiveness of these losses in the context of deep learning.
In deep learning, losses are corrected to be robust to noisy samples, or more exactly, to eliminate the influence of noisy samples.
Based on the estimated noise transition matrix, Forward and Backward \cite{patrini2017making} modify the loss function and build an end-to-end framework.
HOC~\cite{zhu2021clusterability} recently proposes to work on clusterable feature representations so as to efficiently estimate noise transition matrix, and further conduct better loss correction.
Other loss correction strategies include \cite{reed2014training,goldberger2017training}.
Different from these methods, we use the cross-entropy loss albeit adapt it for latent variable models.

\paragraph{Label correction}
JO \cite{tanaka2018joint} is a joint optimization framework where network parameters and class labels are optimized alternatively in training.
Inspired but quite unlike \cite{tanaka2018joint}, rather than correcting labels by using the running average of network predictions, PENCIL \cite{yi2019probabilistic} corrects labels via an updating label distribution in an end-to-end manner. 
Moreover, those noisy labels are only utilized for initializing the label distributions, and 
the network loss function is computed using the label distributions.
SELFIE~\cite{song2019selfie} selects refurbishable samples which are of low uncertainty and can be corrected with a high precision, then replaces their labels based on past model outputs. These corrected samples together with other low-loss instances are later used to update the network.
Another state-of-the-art method named PLC \cite{zhang2021learning} focuses more on feature-dependent label noise where labels are progressively corrected based on the confidence of the noisy classifier.
Notably, we keep using all the labels including those noisy ones instead of conducting label correction which is more complicated.

\paragraph{Semi-supervised methods}
In \cite{ding2018semi}, a two-stage method is proposed where samples are identified as clean or corrupted in the first stage, and then networks are trained in a semi-supervised manner with only the clean labels utilized in stage two.
\cite{kong2019recycling} also conducts a similar two-stage method with Renyi entropy regularization used in stage two.
DivideMix \cite{li2020dividemix} is one of the state-of-the-art methods achieving high accuracy on real noisy datasets.
Specifically, it dynamically divides training data into a labeled clean set and an unlabelled corrupted set, and then trains models on both sets in a semi-supervised manner with improved MixMatch \cite{berthelot2019mixmatch} strategy.
It can be seen that these methods mainly differ in adopting different criteria for semi-supervised learning step after dividing the training set into clean and corrupted subsets.

\section{Method} \label{sec::method}
In this section, we present our approach for learning with noisy labels. 
We start with recalling compression techniques including \textit{Dropout}~\cite{srivastava2014dropout} and its structured variant named \textit{Nested Dropout}~\cite{rippel2014learning} in Section \ref{subsec::nested}.
Next, we combine them with one commonly accepted approach named \textit{Co-teaching}~\cite{han2018co} 
(Section \ref{subsec::co_teaching}) in Section \ref{subsec::combination}.
The reason for this combination will be discussed in detail in Section \ref{subsec::comco}.

\subsection{Compression regularizations} \label{subsec::nested}
Here, we consider two compression regularizations that are plug-and-play modules, which can be inserted into common network architectures.
For the sake of clarity, we summarize some necessary notations here.
Let $\tilde{Z} \in \mathbb{R}^{\text{channels}\times \text{height}\times \text{width}}$ be the hidden feature representation obtained by the feature network $f$, i.e.,~$\tilde{Z}=f(X)$.
Note that we set the number of channels to be $K$ and leave out the rest for simplicity, that is, $\mathbb{R}^{K\times \cdots}$. 
In this paper, we treat compression methods as applying masks to the obtained feature $\tilde{Z}$. 
In this manner, let $M \in \mathcal{M}$ be the feature mask where the space $\mathcal{M}$ can vary for different compression methods.
The feature with mask applied is denoted by $Z = M \odot \tilde{Z}$ where $\odot$ is the element-wise product.
Then, $Z$ will be fed into the subsequent network structures.
The two compression methods are given w.r.t.~their specific mask distributions as follows:
 
\begin{figure}[t]
	\begin{minipage}[t]{0.38\textwidth}  
		\centering  
		\includegraphics[width=\textwidth]{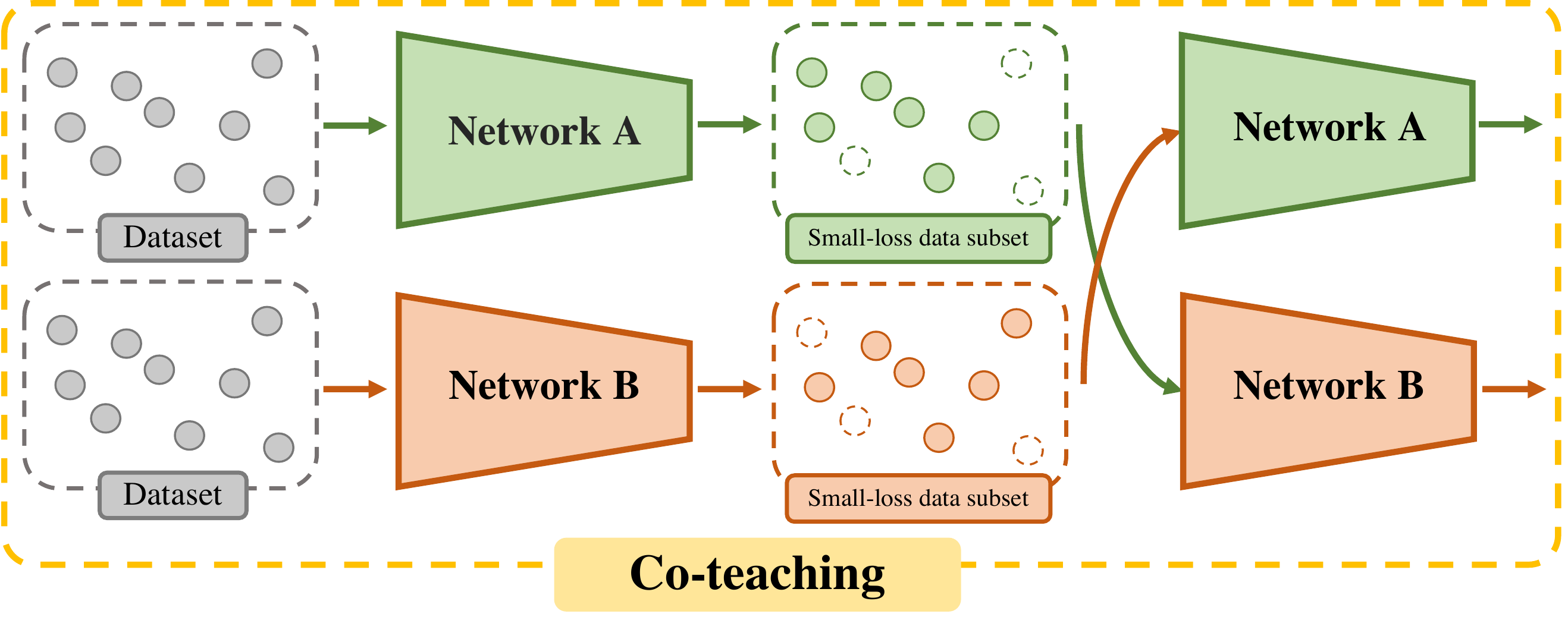}
	\end{minipage}  
	\centering  
	\caption{Co-teaching trains two networks simultaneously where base network updates itself based on the small-loss mini-batch samples selected by its peer.}
	\label{fig::coteaching}
	\vspace{-3mm}
\end{figure}

\subsubsection{Dropout}
Dropout~\cite{srivastava2014dropout} is one classical method for feature compression where each feature in the network layer it applies is dropped according to a Bernoulli distribution.
The space of its feature mask $\mathcal{M}$ is defined by
\begin{align} \label{eq::mask_drop}
    \mathcal{M} := \{ M \in  
    \mathbb{R}^{K\times \cdots}\, | \,
    \forall \, 1 \leq k \leq K, M_k \sim \mathcal{B}(p_{\text{drop}})\}
\end{align} 
where $\mathcal{B}$ is the Bernoulli distribution with $M_k$ being either $\mathbf{1}$ or $\mathbf{0}$, and $p_{\text{drop}}$ is the drop rate. 

\subsubsection{Nested Dropout}
Nested Dropout \cite{rippel2014learning} learns ordered representations with different dimensions having different degrees of importance.
Although it is originally proposed to perform fast information retrieval and adaptive data compression, we find that it can properly regularize a network to combat label noise.
In particular, while Nested Dropout is applied, the meaningless representations can be dropped, which leads to a compressed network~\cite{gomez2019learning}.
Considering above, these ordered representations can be adapted to learning with noisy labels since representations learned from noisy data are supposed to be meaningless. 
Consequently, Nested Dropout may serve as a strong substitute of Dropout.

In order to obtain an ordered feature representation, in each training iteration, we only keep the first $k$ dimensional feature of $\tilde{Z}$ and mask the rest to zeros, that is, $M \in \mathcal{M}$ where
\begin{align} \label{eq::mask_nested}
    \mathcal{M} &:= \{ M \in \mathbb{R}^{K\times\cdots}\, | \, 
    \forall \, k \sim \mathcal{C}(p_1,\ldots,p_K),
    \nonumber\\
    & \forall \, 1 \leq i \leq k, 
    M_i = \mathbf{1} \, 
    \text{and}\,  \forall \, k < i \leq K, M_i = \mathbf{0}\},
\end{align}
with $\mathbf{1}$ and $\mathbf{0}$ being all-ones and all-zeros tensors, respectively.
Moreover, $k$ is sampled from a categorical distribution denoted by $\mathcal{C}$ with corresponding parameters as follows:
\begin{align} 
\label{eq::CatGaussian}
    \Big\{ p_k \propto \exp\Big( -\frac{k^2}{2\, \sigma_{\text{nest}}^2} \Big), \quad \forall k=1,\ldots,K \Big\}
\end{align}
where $\sigma_{\text{nest}}$ is the major hyper-parameter in Nested Dropout.
In this case, smaller $k$ is preferred if $\sigma_{\text{nest}}$ is small.
Moreover, though we could compute $\mathbb{E}_{\mathrm{P}_M}(Z):=\mathbb{E}_{\mathrm{P}_M}(M\odot \tilde{Z})$ with $\mathrm{P}_M$ being \eqref{eq::CatGaussian} exactly during inference, we find it more efficient to verify which $k$ yields the best performance on the validation set, and then keep the model induced by $k$ for testing. 

Note that rather than treating Dropout and Nested Dropout merely as regularizations, we focus on their ability of inducting ensemble of models, i.e.,~\eqref{eq::lvm}, which will be carefully discussed in Section \ref{subsec::overMemo}.
We underline this ensemble property in Fig.~\ref{fig::compress}.

\begin{figure}[t]
	\begin{minipage}[t]{0.45\textwidth}
		\centering  
		\includegraphics[width=\textwidth]{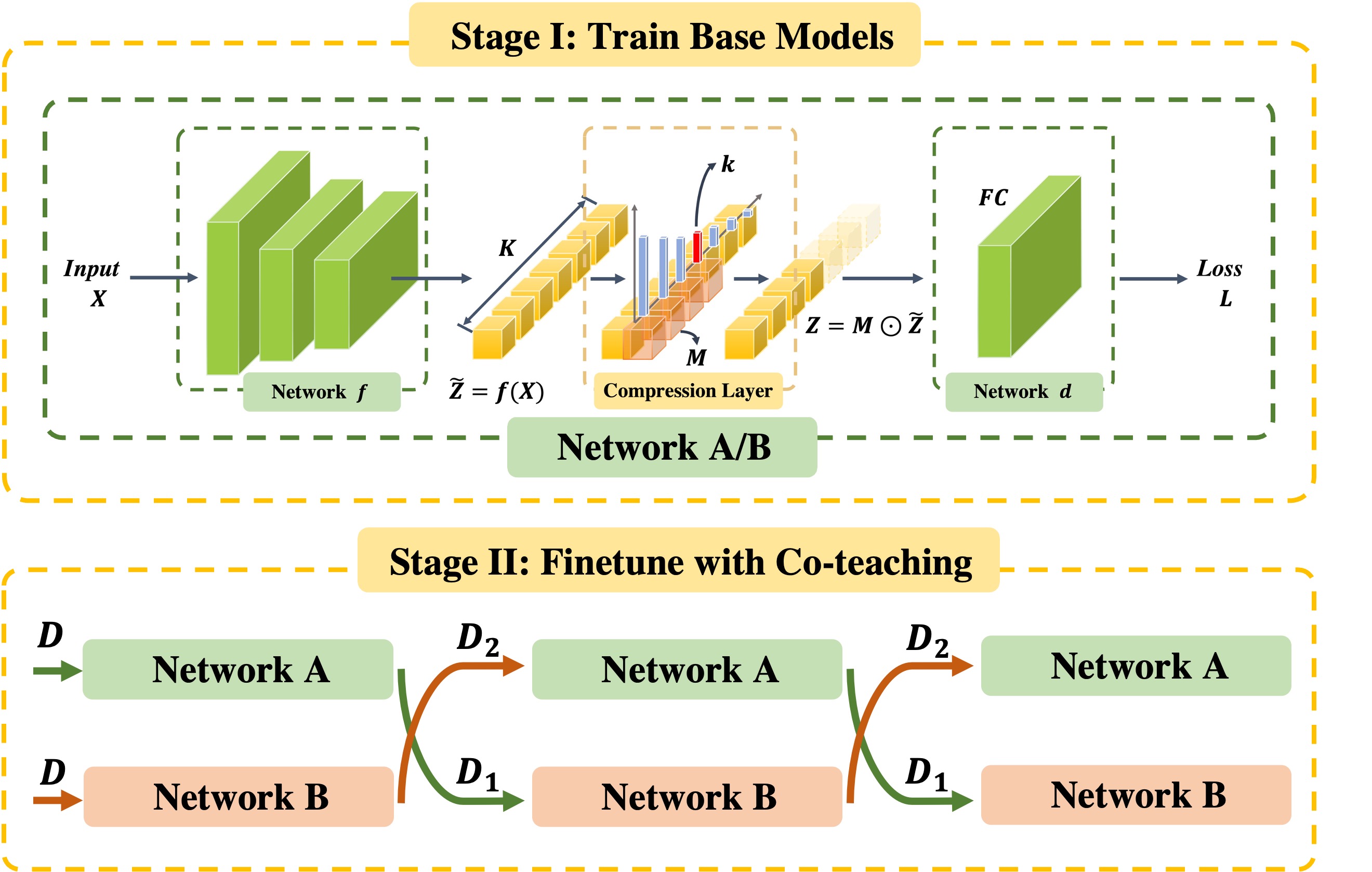}
	\end{minipage}  
	\centering  
    \caption{Overview of our method. In stage one, the hidden activation $\tilde{Z}$ is computed by a feature extractor $f$. 
    Dropout/Nested Dropout is applied to $\tilde{Z}$ by masking some of the features to zeros, i.e.,~$Z=M\odot \tilde{Z}$. 
    The compressed feature $Z$ is then fed into the network structure $d$, which can simply be a fully connected layer (FC), to perform the final prediction. 
    In stage two, the two base networks are fine-tuned with Co-teaching.}
	\label{fig::workflow}
	\vspace{-3mm}
\end{figure}

\subsection{Co-teaching} \label{subsec::co_teaching}
Co-teaching~\cite{han2018co} is a baseline method for learning under label noise.
It trains two deep networks with identical architecture, i.e.,~$h_1$ and $h_2$, simultaneously where each network selects its $100(1 - \lambda_{\text{forget}}$) percent \textit{small-loss} instances, leading to $\mathcal{D}_1$ and $\mathcal{D}_2$ respectively, where
$\lambda_{\text{forget}}$ is the forget rate.
Note that $\lambda_{\text{forget}}$ is a crucial hyper-parameter in the Co-teaching architecture.
Networks update themselves basing on the data subset selected by their peers.
We provide a illustration for clarity in Fig.~\ref{fig::coteaching}.

Co-teaching bases on the concept that small-loss instances are more likely to be clean \cite{han2018co, jiang2018mentornet, yu2019does, tanaka2018joint, kumar2010self}.
Therefore, classifiers trained on them are supposed to be more resistant to noisy labels.
However, one non-negligible premise is that base models should be reliable enough to select samples which are indeed clean.
To prevent constantly bad selections, a scheduling has been proposed in \cite{han2018co}.
That is, Co-teaching first keeps all the samples in the mini-batch, then gradually decreases the sample size in $\mathcal{D}_1$ and $\mathcal{D}_2$ till the predefined $N$-th epoch, after which the sample size used for training kept fixed.
Nevertheless, we experimentally find that the tuning of $N$ is not stable since $N$ varies with different levels of label noise.
Therefore, rather than training Co-teaching with random initialized base models and tuning on $N$, we employ well-trained models as initialization for better and stable performance.

\begin{algorithm}[t]
	\caption{Two-stage compression training}
	\label{alg::twostage}
	\KwIn{
		training data $D$ with size $|D|$, 
		compression hyper-parameters $\{\sigma_{\text{nest}}, p_{\text{drop}}\}$,
		two initialized networks $h_1$, $h_2$,
		loss $L_q$ \eqref{eq::newLoss},
		forget rate $\lambda_{\text{forget}}$.
	}   
	{\bf Ensure:} Either train with Dropout ($p_{\text{drop}}>0$) by \eqref{eq::mask_drop}, \\
	\quad \quad \quad \, or Nested Dropout ($\sigma_{\text{nest}}>0$) by \eqref{eq::mask_nested}, \eqref{eq::CatGaussian}.

	\While{$h_1$, $h_2$ not converge}{ 
        Train $h_1$, $h_2$ independently on $D$ with loss $L_q$ under
        (Nested) Dropout;
	}
	
	\While{Fine-tune with Co-teaching}{
	Randomly separate mini-batch $D_{m}$ into two subsets: $D_{m_1}$, $D_{m_2}$ with $|D_{m_1}|=|D_{m_2}|$;\\
	$h_1$ selects $(1-\lambda_{\text{forget}})|D_{m_1}|$ small-loss data $\tilde{D}_{m_1}$;\\
	$h_2$ selects $(1-\lambda_{\text{forget}})|D_{m_2}|$ small-loss data $\tilde{D}_{m_2}$;\\
	Train $h_1$ on $\tilde{D}_{m_2}$, $h_2$ on $\tilde{D}_{m_1}$ independently with loss $L_q$ under (Nested) Dropout;\\
	}
	\KwOut{$(h_1 + h_2)/2$}
\end{algorithm}

\subsection{Combination} \label{subsec::combination}
Now we combine the compression regularizations with Co-teaching in a two-stage manner:
\begin{itemize}
    \item Train two \textit{Dropout} / \textit{Nested Dropout} networks separately until convergence;
    \item Fine-tune these two networks with \textit{Co-teaching}.
\end{itemize}

Co-teaching is chosen for fine-tuning since its cross-update mechanism would help in alleviating the over-fitting issue over a single model.
As mentioned in \cite{han2018co}, different classifiers are able to generate different decision boundaries and then have different abilities to learn.
During training on noisy labels, we expect the two neural networks to adaptively compress out the noisy information left by their peer networks where samples with obviously corrupted labels have been already excluded.
Hence, the base networks are less likely to overfit the corrupted labels. 
The above idea is validated in Section \ref{subsec::comco} with the help of a \textit{bias-variance decomposition} for Co-teaching.

We now specify our method.
In the first stage, two networks are trained independently until convergence so as to provide better base models for Co-teaching.
Moreover, a learning rate warm-up is set to cope with the difficulty of training with Dropout/Nested Dropout in the early epochs, which results from the high probability of dropping most of the channels in the feature layer when $p_{\text{drop}}$ is large or $\sigma_{\text{nest}}$ is small.
In the second stage, since we only fine-tune the networks with Co-teaching, Dropout/Nested Dropout is maintained during the training of each model except for the selection procedure of small-loss data subsets $\mathcal{D}_1$, $\mathcal{D}_2$. 
Note that the performance of Co-teaching also depends on the diversity of the base models.
In this case, we modify the original Co-teaching \cite{han2018co} with batch separation strategy where each batch is divided into two data subsets with equivalent size, and small-loss data selections are then conducted on these two subsets separately.
The final result is the accuracy of the ensembled model. 
The workflow of our two-stage method is given in
Fig.~\ref{fig::workflow} and Algorithm \ref{alg::twostage} for clarity.

\section{Theoretical analysis} \label{sec::theory}
This section provides the motivation and validity of our method.
Notations and basic concepts in need are given in Section \ref{subsec::preliminaries}.
We state that the key issue in combating label noise is to prevent networks from over-fitting based on a bias-variance decomposition in Section \ref{subsec::overMemo}.
To this end, we recall two compression regularizations which can be treated as implicit information bottleneck in Section \ref{subsec::compression}.
More on Nested Dropout is in Section \ref{subsec::nested_ranking}.
Finally, we verify that the Co-teaching combination leads to even less over-fitting based on the bias-variance decomposition in \ref{subsec::comco}.
For the sake of clarity, all the proofs in this section are given in the Appendix.

\subsection{Preliminaries} \label{subsec::preliminaries}
First of all, we formulate the problem of learning with noisy labels.
Let $X \in \mathcal{X}$ be the input variable where $\mathcal{X}$ is the input feature space.
We consider the data generation process for the training set:
\begin{align} \label{eq::generation}
    x \sim p(x), \quad \varepsilon \sim p(\varepsilon), \quad 
    y \sim p(y | x, \varepsilon),
\end{align}
where $\varepsilon$ is the noise occurred during labelling. 
In this manner, we denote by $Y \in \mathcal{Y}$ the contaminated label where $\mathcal{Y}$ is the corresponding signal space.
The goal is to learn a model on the corrupted dataset 
for testing on clean data drawn from the same generative process expect that $\epsilon \approx 0$.

Next, we cover some basic concepts in information theory \cite{shannon1948mathematical}.
The \textit{entropy} gives the amount of information coded in a distribution or equivalently the uncertainty about a random variable, and it is defined as the average code length:
\begin{align*}
    H_p(Y) := - \int_{\mathcal{Y}} p(y) \log p(y) \, dy
\end{align*}
where $p(y)$ is a fixed probability measure on $\mathcal{Y}$.
Similarly, \textit{conditional entropy} gives the amount of information about one random variable given another random variable:
\begin{align*}
    H_p(Y | X) := - \int_{\mathcal{Y}} \int_{\mathcal{X}} p(x, y) \log p(y | x) \, dx dy
\end{align*}
where $p(x, y)$ is the joint probability measure on $\mathcal{X} \times \mathcal{Y}$ and $p(y|x)$ is the conditional p.d.f.
The \textit{cross entropy} with respect to a model distribution $q(y|x)$ is defined by
\begin{align} 
    H_{p,q} (Y|X)
    & := \mathbb{E}_{p(x)} \mathbb{E}_{p(y|x)} [- \log q(y|x)] 
    \nonumber \\
    & = \mathbb{E}_{p(x)} \mathbb{E}_{p(y|x)} \bigg[- \log p(y|x) + \log \frac{p(y|x)}{q(y|x)}\bigg]
    \nonumber \\
    & \geq H_p(Y|X),\label{eq::CE2}
\end{align}
which upper bounds the conditional entropy as in \eqref{eq::CE2}. 
Note that the inequality holds for that the Kullback–Leibler divergence, i.e.,~the second term, is always non-negative.
A related quantity is the \textit{cross-entropy} loss $-\log q(y|x)$.
The \textit{mutual information} measures the statistical dependency between random variables $X$ and $Y$ by comparing their joint density with the product of each marginal density:
\begin{align*}
    I(X; Y) := \int_{\mathcal{Y}} \int_{\mathcal{X}} p(x, y) \log \frac{p(x,y)}{p(x) p(y)} \, dx dy.
\end{align*}
Note that since the mutual information is a function of $p(x, y)$, we modify the notation to $I_p(X; Y)$ for better emphasizing on the actual variable.
Moreover, given a factorization $p(x,y) = p(x) p(y|x)$, we have
\begin{align*}
    I_p(X; Y) = H_p(Y) - H_p(Y|X)
\end{align*}
which leads to another interpretation, that is, the reduction in uncertainty of $Y$ by knowing $X$.
In addition, the \textit{conditional mutual information} is defined as follows:
\begin{align*}
    I_p(X; Y|Z) := \int_{\mathcal{Z}} \int_{\mathcal{Y}} \int_{\mathcal{X}} p(x,y,z) \log \frac{p(x,y|z)}{p(x|z)p(y|z)} \, dx dy dz.
\end{align*}

Here and subsequently, we let $\mathbb{E}_{p(x)}[x] \equiv \mathbb{E}_{\mathrm{P}_X}[X]$ stand for the expected value of $X$ and $D_{\text{KL}}(\cdot)$ for the Kullback–Leibler divergence.
Besides, we denote the capital Roman alphabet for random variables or matrices and their lowercase for the values.
Moreover, we denote by $Z_i$ the $i$-th row if $Z$ is a matrix or the $i$-th channel if $Z$ is a $3$-dimensional tensor. 
We also write $Z_{i:j}$ for the slice from the $i$-th channel to the $j$-th channel, and $\odot$ for the element-wise multiplication.

\subsection{Bias-variance decomposition for noisy labels} \label{subsec::overMemo}
Considering that deterministic networks are likely to overfit the noisy training set, we introduce \textit{an ensemble of models} and rely on the intersection of these models to extract consistent information.
The idea is that the information learned from the noise are less likely to be consistent across different models.
This motivates us to consider the \textit{latent variable model} since it can be treated as an ensemble of models:
\begin{align} \label{eq::lvm}
    q(y|x) := \int_{\mathcal{Z}} q(y|z) q(z|x) \, dz,
\end{align}
where $q(z|x)$ is the encoder, and $q(y|z)$ is the decoder or can even be an individual model induced by a particular instance of $z$.
For practical reasons, we would like to use existing network architectures, such as ResNet \cite{he2016deep}, to construct the encoder $q(z|x)$ and the decoder $q(y|z)$.
Our strategy is to split an entire network architecture, e.g.~a ResNet-18, into two parts as shown in Fig.~\ref{fig::workflow}.
In this case, it is natural to take the second part plus a softmax layer to implement $q(y|z)$.
The first part is however insufficient to implement $q(z|x)$ as we will discuss later.

Since the cross-entropy loss $-\log q(y|x)$ is intractable due to the integral in \eqref{eq::lvm}, we consider a surrogate quantity to $q(y|x)$ using Jensen's inequality:
\begin{align} \label{eq::surrogate}
    \tilde{q}(y|x) \propto \exp \big[\mathbb{E}_{q(z|x)} \log q(y|z)\big]
    \leq q(y|x)
\end{align}
and therefore we define the new loss function as the negative log-likelihood with respect to $\tilde{q}(y|x)$:
\begin{align} \label{eq::newLoss}
    L_q(x,y) := \mathbb{E}_{q(z|x)} \big[-\log q(y|z) \big] \propto -\log \tilde{q}(y|x).
\end{align}

Based on \eqref{eq::generation}, \eqref{eq::surrogate} and \eqref{eq::newLoss}, we now derive the bias-variance decomposition of the proposed latent variable with the new loss function under label noise as follows:
\begin{theorem} \label{thm::decomp}
Let $x,y\sim p(x,y) = p(x)p(y|x)$ where $p(y|x) = \int p(y|x,\varepsilon) p(\varepsilon)\, d\varepsilon$, for loss $L_q(x,y)$ defined in \eqref{eq::newLoss}, the risk has a bias-variance decomposition:
\begin{align} \label{eq::bvDecompo}
    & \mathbb{E}_{p(x,y)}\big[ L_q(X,Y)\big] 
    = \mathbb{E}_{p(x)} \Big[ \underbrace{D_{\text{KL}}\big( p(y|x) \| \tilde{q}(y|x) \big) }_{\text{bias}}  
    \nonumber\\
    & \qquad + \underbrace{\mathbb{E}_{q(z|x)} D_{\text{KL}}\big( \tilde{q}(y|x) \| q(y|z) \big)}_{\text{variance}} \Big] + \text{const}
\end{align}
where $\tilde{q}(y|x) \propto \exp \big[\mathbb{E}_{q(z|x)} \log q(y|z)\big]$ is the average or an ensemble of models.
\end{theorem}

Intuitively, the bias term in \eqref{eq::bvDecompo} determines how close the average model $\tilde{q}(y|x)$ is to $p(y|x)$ and $p(y|x)$ is the conditional probability for the noisy $Y$, while the variance term promotes a consensus among individual models.
The variance term also serves as a regularization to combat label noise in the sense that the consensus downweights the influence of the incorrect labels.
Unlike learning with clean data, we do not expect low bias as it indicates model's over-fitting to label noise.
Instead, we rely on the variance term and early stopping to provide good training signals.
In regard of above, the problem of learning with noisy labels can be simplified to \textit{how can we prevent models from over-fitting the noisy training labels?}

\begin{figure}[t]
	\begin{minipage}[t]{0.45\textwidth}  
		\centering  
		\includegraphics[width=\textwidth]{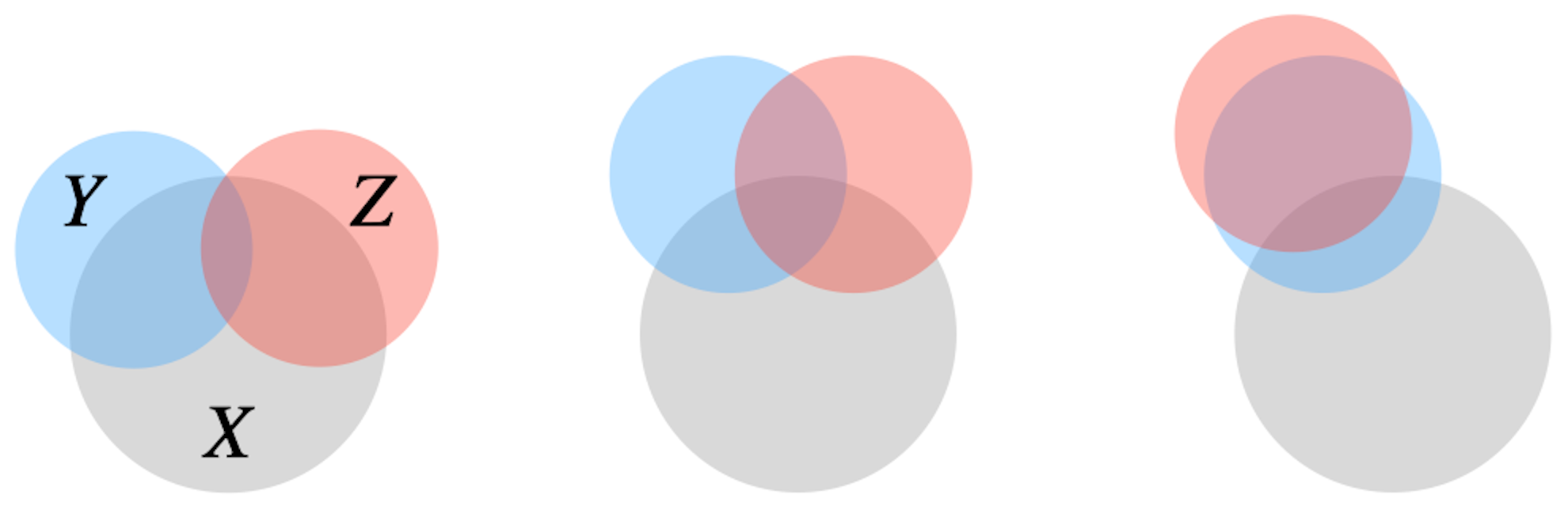}
	\end{minipage}  
	\centering  
	\caption{Information diagrams for the relationship among three random variables.
	Each circle represents the entropy of the corresponding random variable. The intersection between two circles is the mutual information.
	\textbf{Left}: if $Y \rightarrow X \rightarrow Z$ holds. \textbf{Middle}: the general case. \textbf{Right}: an over-fitting case where $I(Y; Z)$ is large but $I(X; Z)$ is small.}
	\label{fig::infodiagram}
	\vspace{-3mm}
\end{figure}
It is worth mentioning that a careful design of the encoder $q(z|x)$ is necessary in order to make better use of the variance term.
For instance, if we choose $q(z|x)$ to be a Dirac delta function, i.e.,~a deterministic mapping from $x$ to $z$, the variance term is zero.
As a result, the training will focus on minimizing the bias term, leading to an easy over-fitting to the distribution of noisy labels.
On the other hand, if the variance is too large, there will be little consensus among individual models, and therefore, no consistent information could be learned by the latent variable model.
Thus in Section \ref{subsec::compression}, we propose to design $q(z|x)$ by incorporating a compression inductive bias for better combating label noise.

\subsection{Compression regularizations} \label{subsec::compression}
We would like to create an information bottleneck for the latent variable model. 
In this manner, the noisy information can be filtered out systematically.
To be specific, we propose to create an information bottleneck by masking and damping the output of the feature extractor $f(X)$:
\begin{align}
    & \min_{p} \, -I_p(Y;Z)
    \label{eq::newIB}\\
    & \text{s.t.} \quad Z = M \odot f(X),\, M \sim \mathrm{P}_M \nonumber
\end{align}
where the distributions of $M$, i.e.,~$\mathrm{P}_M$, is an extrinsic source of randomness, which is tuned on a held-out clean dataset.
Here $M$ is also called mask as defined in \eqref{eq::mask_drop} and \eqref{eq::mask_nested}.

However, here comes the question that \textit{why don't we use Tishby's information bottleneck \cite{tishby2000information} directly for compression?}
As mentioned in \cite{tishby2000information}, the relationship among the input $X$, the label $Y$ and the feature representation $Z$ in the network is given by a Markov chain:
$Y \to X \to Z$.
In this consideration, \cite{tishby2000information} proposes to learn a good feature representation by minimizing the weighted sum of the data fitting term $-I_p(Y; Z)$ and the complexity term $I_p(X;Z)$ with respect to the distribution $p(x,y,z)$, that is
\begin{align} \label{eq::IB}
    \min_{p} \, -I_p(Y;Z) + \beta I_p(X;Z).
\end{align}
However, when learning with noisy labels, there is no clear causal relationship between $Y$ and $X$.
Therefore, Tishby's information bottleneck principle cannot be applied in this case.
To be more specific, we may argue in terms of the information diagrams shown in Fig.~\ref{fig::infodiagram}.
The left image in Fig.~\ref{fig::infodiagram} visualizes the Markov chain $Y \to X \to Z$, which is called the \textit{Mickey Mouse I-diagram} in \cite{kirsch2020unpacking} since $I(Y;Z|X) = 0$, which implies that $I(X;Y;Z) = I(Y;Z)$.
In this case, $I(Y;Z)$ is always smaller than $I(X;Z)$, hence we can prevent $Z$'s over-fitting to $Y$ by reducing $I(X;Z)$.
However, in general, $I(Y;Z|X) \neq 0$ as shown in the middle in Fig.~\ref{fig::infodiagram}, and we can find cases where $I(X;Z)$ is small but $Z$ overfits label noise.
To conclude, the traditional information bottleneck \eqref{eq::IB} may not be effective when dealing with label noise.

Comparing to Tishby's IB \eqref{eq::IB}, we discard the term $I(X;Z)$ in \eqref{eq::newIB} completely and rely only on the held-out clean dataset to remove noisy information.
Since $I_p(Y;Z)$ is intractable, we further adjust it to be computationally available.
Instead of estimating the joint distribution $p(x,y,z)$, we consider a surrogate joint distribution $q(x,y,z) = p(x,y)q(z|x) \approx p(x,y,z)$, and access to $p(x,y)$ only through its samples.
Note that the idea is similar to the variational information bottleneck by \cite{alemi2016deep}.
Specifically, we first identify that $-I_p(Y;Z) = H_p(Y|Z) - H_p(Y)$ where $H_p(Y)$ is a constant for the representation learning, and hence we approximate $H_p(Y|Z)$ by
\begin{align*}
    & H_p(Y|Z) \leq H_{p,q} (Y|Z) = \mathbb{E}_{p(x,y,z)}[-\log q(y|z)]
    \nonumber \\
    & \approx \mathbb{E}_{p(x,y)} \mathbb{E}_{q(z|x)} \big[ - \log q(y|z)\big]
    = \mathbb{E}_{p(x,y)} \big[ L_q (X, Y) \big].
\end{align*}

Now, we rewrite our learning with noisy labels as follows:
\begin{align}
    & \min_{q} \, \mathbb{E}_{p(x,y)} \big[ L_q (X, Y) \big]
    \nonumber \\
    & \text{s.t.} \quad Z = M \odot f(X),\, M \sim \mathrm{P}_M
    \label{eq::mask}
\end{align}
where $L_q(x,y) := \mathbb{E}_{q(z|x)} \big[-\log q(y|z) \big]$ as in \eqref{eq::newLoss}.
As such, we only need to learn the proposed $q(z|x)$ and $q(y|z)$.
In particular, we propose an implicit parameterization for $q(z|x)$ as specified in \eqref{eq::mask}, which involves a feature extractor $f(X)$ and an external random variable $M$.
We also proposed to learn $q(z|x)$ and $q(y|z)$ jointly by minimizing $\mathbb{E}_{p(x,y)} \big[ L_q (X, Y) \big]$, which will be optimized by stochastic gradient descent (SGD) as we only have access to the samples of $X$, $Y$ and $Z$.

\subsubsection{Dropout}
If we set $\mathrm{P}_M$ to a Bernoulli distribution as in \eqref{eq::mask_drop}, then \eqref{eq::mask} exactly covers Dropout.
However, different from the original formulation in \cite{srivastava2014dropout}, we formulate Dropout into the framework of latent variable models with our loss \eqref{eq::newLoss}.
In this manner, Dropout is the baseline of incorporating compression inductive bias for combating label noise.

\subsubsection{Nested Dropout}
If we specify $M\in\mathcal{M}$ where $\mathcal{M}$ is \eqref{eq::mask_nested}, and 
$\mathrm{P}_M$ as \eqref{eq::CatGaussian}, then \eqref{eq::mask} exactly covers the Nested Dropout.
More properties are provided in the next subsection.

\begin{figure}[t]
	\begin{minipage}[t]{0.4\textwidth}  
		\centering  
		\includegraphics[width=\textwidth]{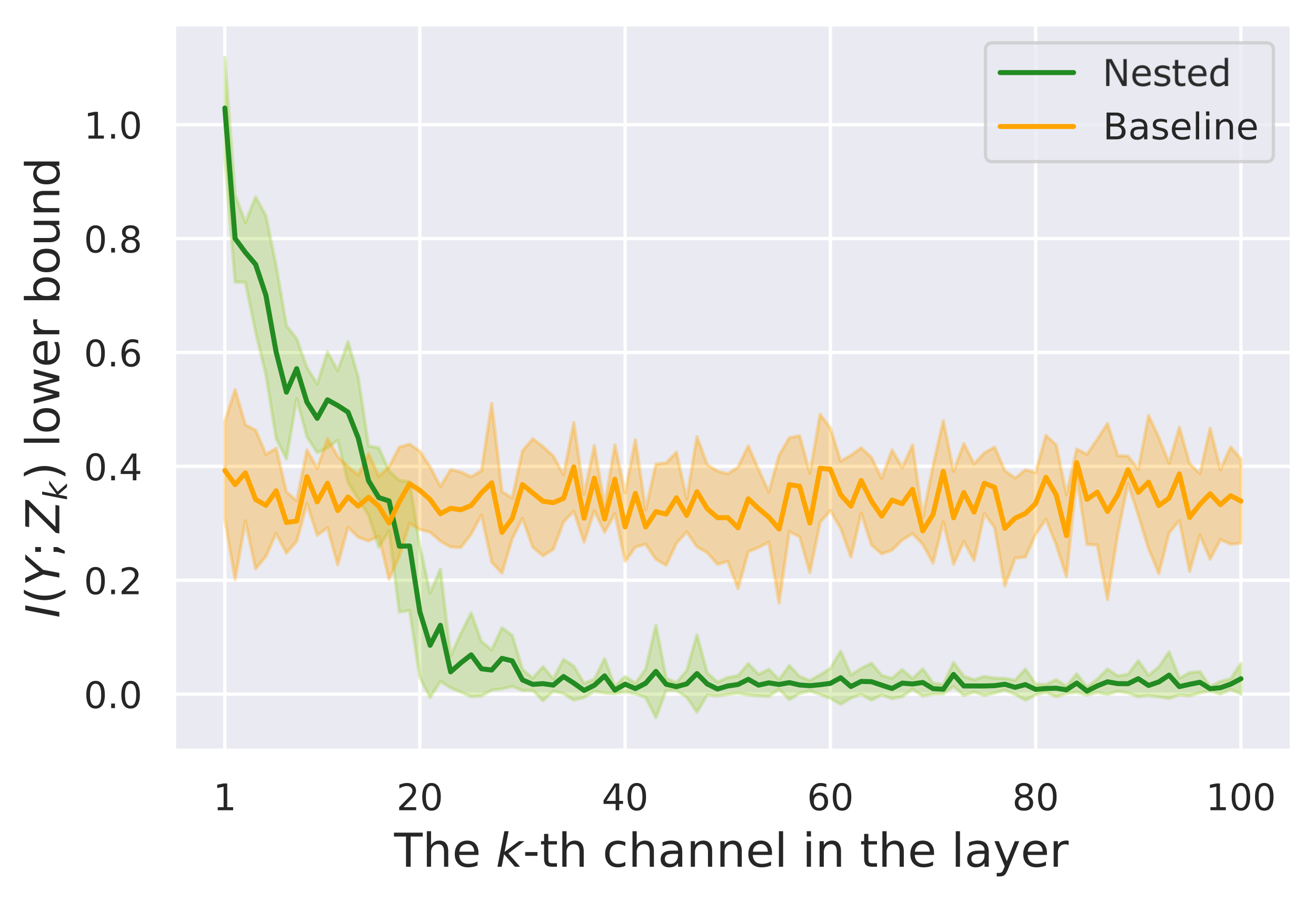}
	\end{minipage}  
	\centering  
	\caption{The mutual information of different channels of models trained on the classification task on CIFAR-10 \cite{krizhevsky2009learning}. 
	The result is the average after $3$ runs with $\pm 95\%$ confidence interval.
	We compare the model trained with and without Nested Dropout. Baseline refers to the one without Nested Dropout.}
	\label{fig::MI}
	\vspace{-3mm}
\end{figure}

\subsection{Nested Dropout} \label{subsec::nested_ranking}
Nested Dropout is a variant of Dropout where the importance of each feature channel is sorted from high to low, while channels in Dropout model are of equal importance.
For better understanding of Nested Dropout's sorting property, we theoretically work on it through mutual information.

Before presenting the theorem on the information sorting property, we need the following Assumption \ref{ass::permu} where hidden features are supposed to be exchangeable \cite{aldous1985exchangeability}.

\begin{assumption} \label{ass::permu}
    Let $X$ be the input, $f$ be the feature extractor, $d$ be the subsequent network structure including the classification head, and $\tilde{Z}=f(X)\equiv [\tilde{Z}_k]_{k=1}^K$ be the hidden feature representation.
    The hidden feature representation $\tilde{Z}$ is exchangeable, and $d$ is also exchangeable.
    That is, for any permutation $\pi$, the model satisfies that
    \begin{align}
        (\tilde{Z}_1,\tilde{Z}_2,\ldots,\tilde{Z}_K)
        &\overset{D}{=} 
        (\tilde{Z}_{\pi_1},\tilde{Z}_{\pi_2},\ldots,\tilde{Z}_{\pi_K}), \label{eq::permu}
        \\
        Y|\tilde{Z}_1,\tilde{Z}_2,\ldots,\tilde{Z}_K
        &\overset{D}{=} 
        Y|\tilde{Z}_{\pi_1},\tilde{Z}_{\pi_2},\ldots,\tilde{Z}_{\pi_K}, \label{eq::dec_permu}
    \end{align}
    where $\overset{D}{=}$ denotes equivalence in distribution.
\end{assumption}

We find this can serve as a valid assumption considering the network architectures.
According to de Finetti's theorem \cite{heath1976finetti}, a sequence of random variables $(\tilde{Z}_1,\tilde{Z}_2,\ldots)$ is infinitely exchangeable iff,
$p(\tilde{z}_1,\tilde{z}_2,\ldots,\tilde{z}_n) = \int \prod_{i=1}^n p(\tilde{z}_i|\theta)P(d\theta)$, for all $n \in \mathbb{N}_{+}$ and some measure $P$ on $\theta$.
If we consider a simple MLP where $\mathbf{\tilde{z}}=f(\mathbf{x})=\mathbf{W}_1\mathbf{x} + \mathbf{b}_1$, $\hat{\mathbf{y}}= d(\mathbf{\tilde{z}}) = \mathbf{W}_2\mathbf{\tilde{z}} + \mathbf{b}_2$ and denote $\theta=[\mathbf{W}_1, \mathbf{b}_1, \mathbf{x}]$.
In this way, for any permutation of the hidden feature $\mathbf{\tilde{z}}_{\pi}:=\pi(\mathbf{\tilde{z}})$, it can be obtained by $\pi(\mathbf{W}_1) \pi(\mathbf{x}) + \pi(\mathbf{b}_1)$.
As it requires to integrate all the possible $\theta$ to obtain $p(\mathbf{\tilde{z}})$ and $p(\mathbf{\tilde{z}}_{\pi})$,
we then have $p(\mathbf{\tilde{z}})=p(\mathbf{\tilde{z}}_{\pi})$.
Since $\hat{\mathbf{y}}_{\pi} = \pi(\mathbf{W}_2) \mathbf{\tilde{z}}_{\pi} + \pi(\mathbf{b}_2) = \hat{\mathbf{y}}$, then $d$ is exchangeable.
Moreover, similar argument can be derived by considering the features after global average pooling in CNNs.

\begin{theorem}  \label{thm::ranking}
    Let $X$ be the input, $f$ be the feature extractor, and $\tilde{Z}=f(X)\equiv [\tilde{Z}_k]_{k=1}^K$ be the hidden feature representation.
    Suppose that the model satisfies Assumption \ref{ass::permu}.
    Then, we have for $Z_k = M_k \odot \tilde{Z}_k$,
    \begin{align}
        &I(Y; \tilde{Z}_1) = \cdots = I(Y; \tilde{Z}_K), \label{eq::IYH} \\
        &I(Y; Z_1) \geq \cdots \geq I(Y; Z_K) \label{eq::IYZ}.
    \end{align}
\end{theorem}

This sorting property is also discussed in \cite{rippel2014learning}, although from the perspective of $I(Y; Z_{1:k}) \leq I(Y; Z_{1:(k+1)})$ with $Z_{1:k}:=[Z_1,Z_2,\ldots,Z_k]$.
In addition to the theoretical analysis, we conduct an experiment on CIFAR-10 \cite{krizhevsky2009learning} using ResNet-18 to verify Theorem \ref{thm::ranking} empirically. 
We plot the empirical estimate of the variational lower bound of $I(Y; Z_k)$, i.e.,~$H(Y | Z_k)$, for $Z$ computed by \eqref{eq::mask} and $Z = \tilde{Z} = f(X)$. 
The comparison is in Fig.~\ref{fig::MI} where the one without Nested Dropout is tagged as baseline. 
A clear information sorting has been achieved comparing to the baseline training.

\begin{figure*}[t]
	\begin{minipage}[t]{0.24\textwidth}  
		\centering 
		\addtocounter{figure}{-1}
		\includegraphics[width=\textwidth]{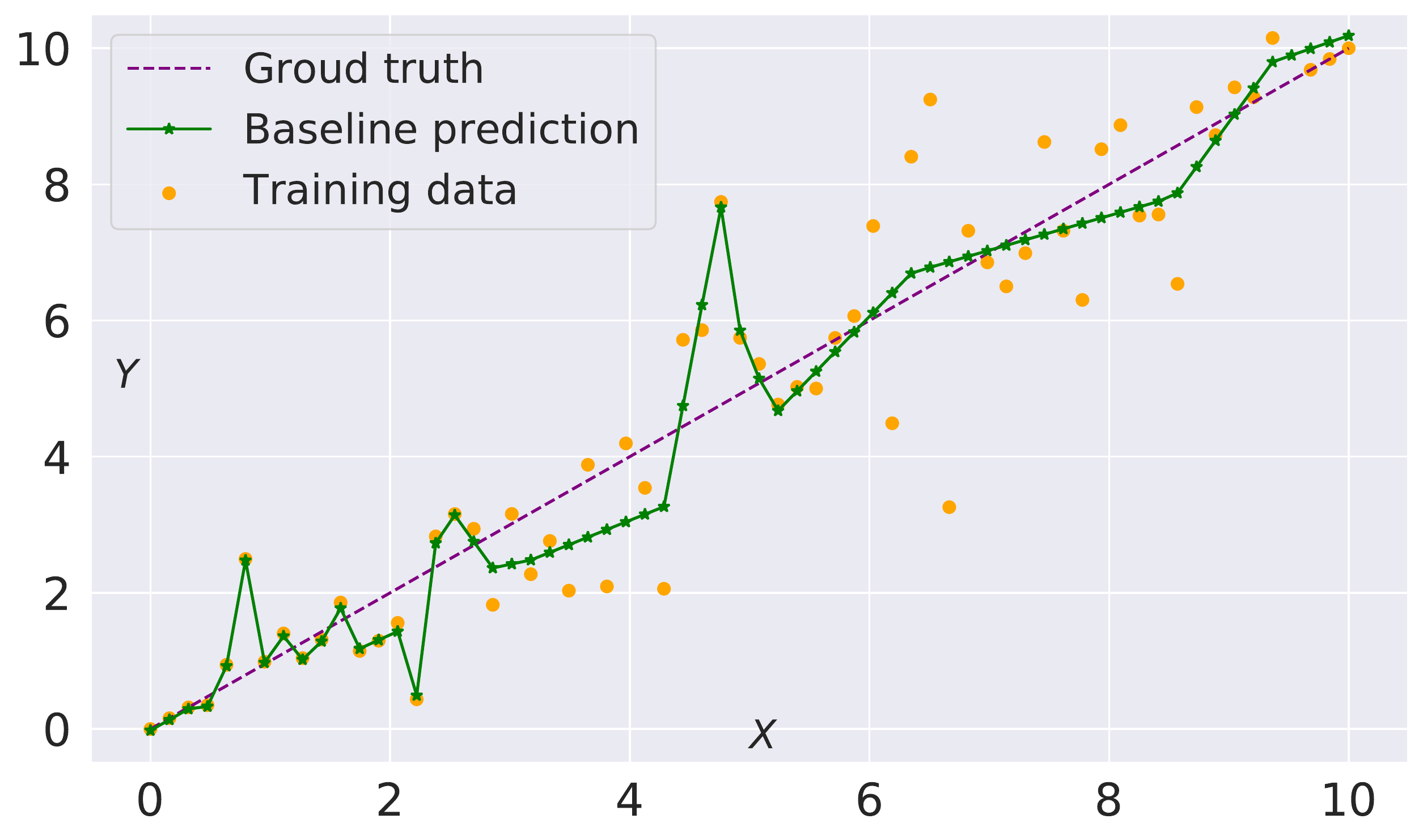}
        \captionsetup{labelformat=empty}
		\caption{(a) MLP}
	\end{minipage}  
	\begin{minipage}[t]{0.24\textwidth}  
		\centering  
		\addtocounter{figure}{-1}
		\includegraphics[width=\textwidth]{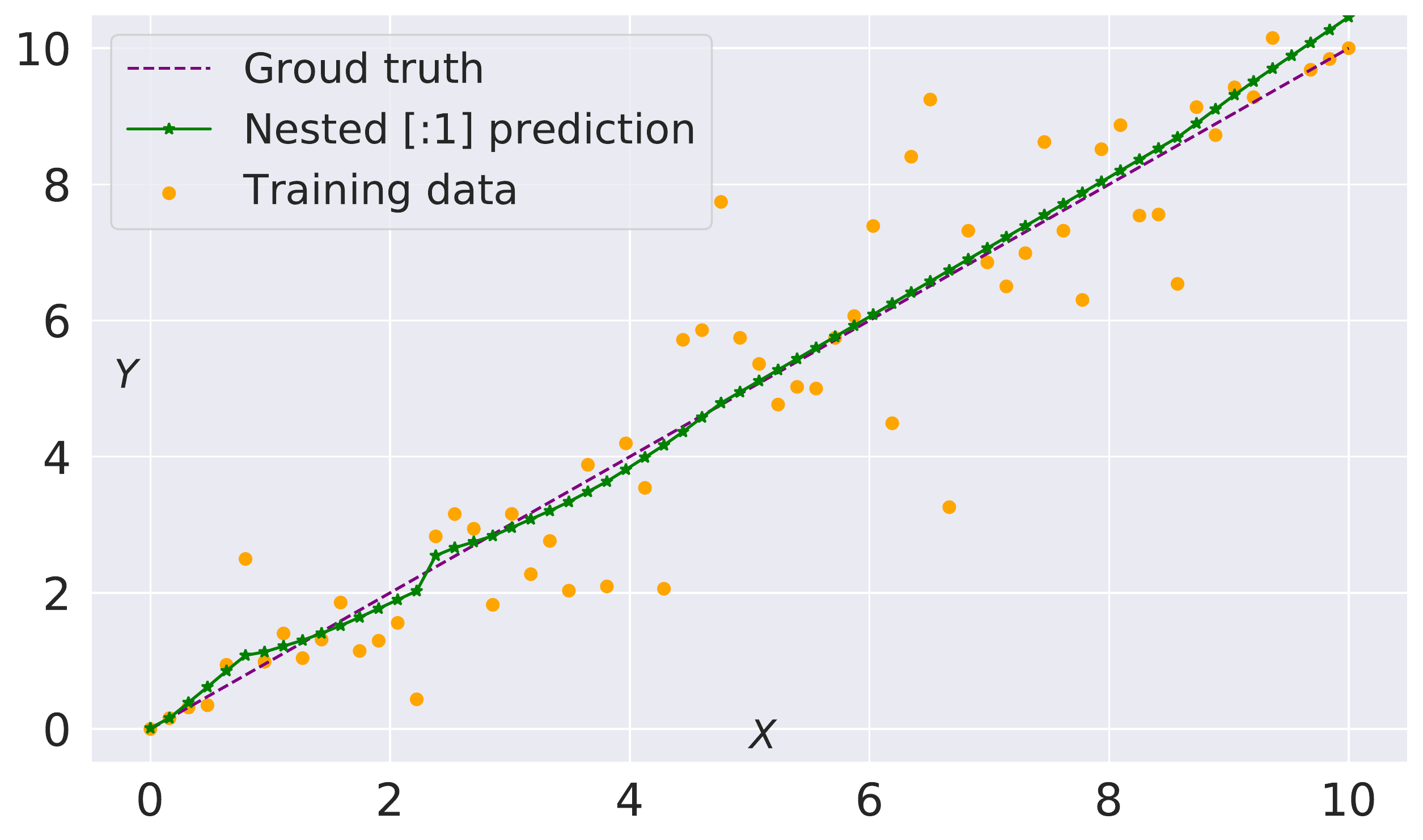} 
		\captionsetup{labelformat=empty}
        \caption{(b) MLP$+$Nested $k$=1}
	\end{minipage} 
	\begin{minipage}[t]{0.24\textwidth}  
		\centering  
		\addtocounter{figure}{-1}
		\includegraphics[width=\textwidth]{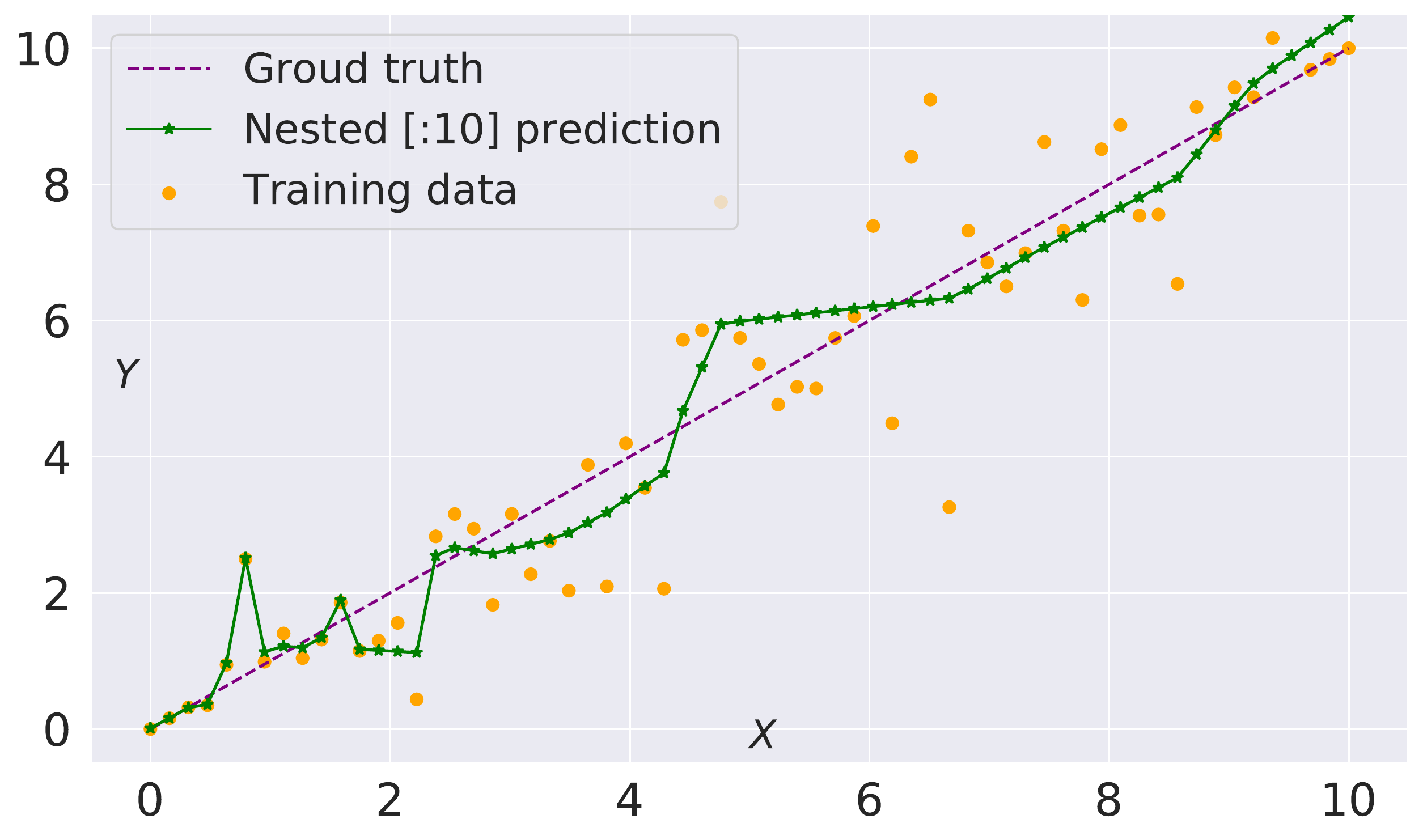} 
		\captionsetup{labelformat=empty}
        \caption{(c) MLP$+$Nested $k$=10}
	\end{minipage} 
	\begin{minipage}[t]{0.24\textwidth}  
		\centering  
		\addtocounter{figure}{-1}
		\includegraphics[width=\textwidth]{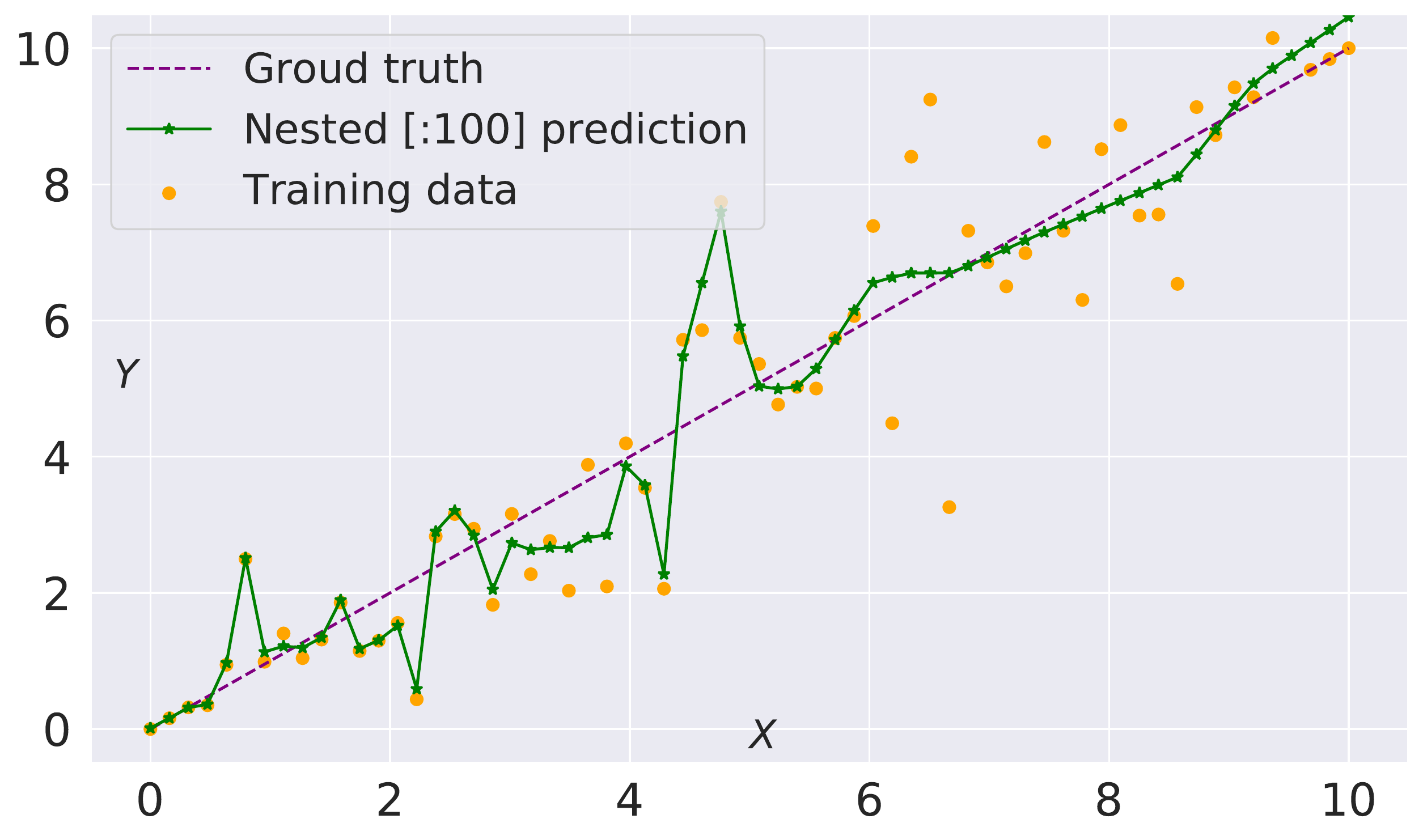}
        \captionsetup{labelformat=empty}
		\caption{(d) MLP$+$Nested $k$=100}
	\end{minipage} 
	\begin{minipage}[t]{0.24\textwidth}  
		\centering  
		\addtocounter{figure}{-1}
		\includegraphics[width=\textwidth]{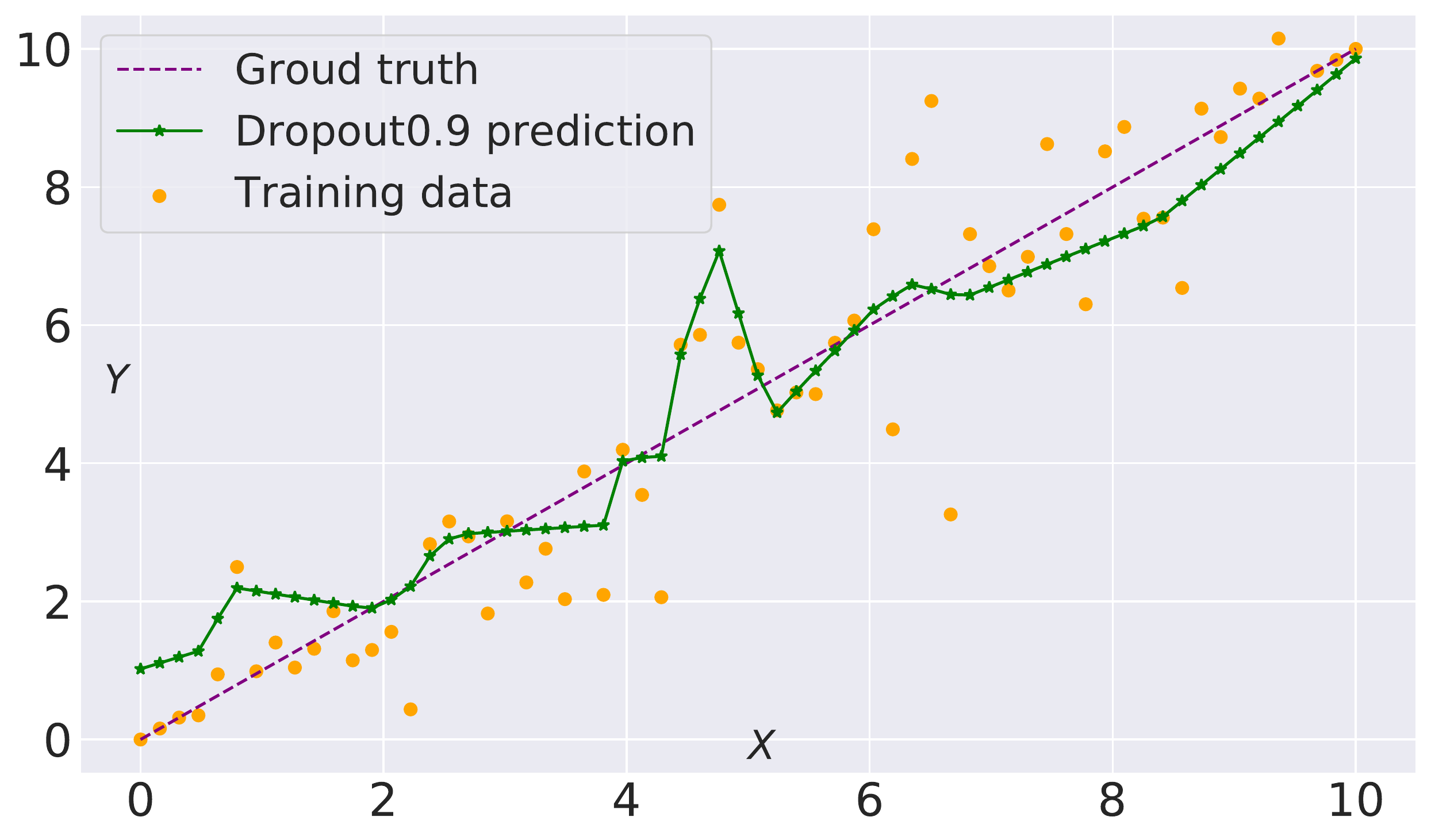}
        \captionsetup{labelformat=empty}
		\caption{(e) MLP$+$Dropout 0.9}
	\end{minipage}
	\begin{minipage}[t]{0.24\textwidth}  
		\centering  
		\addtocounter{figure}{-1}
		\includegraphics[width=\textwidth]{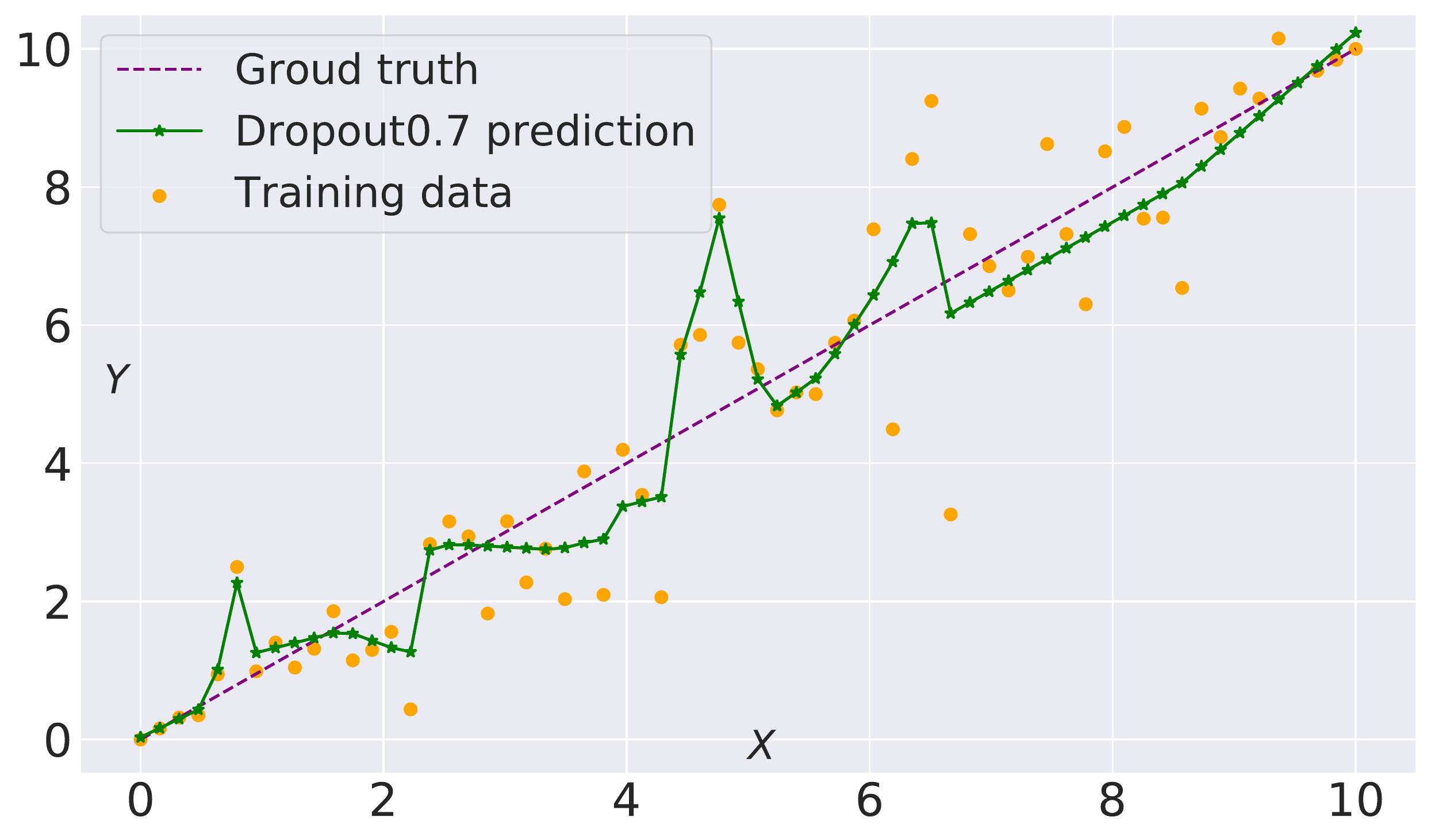}
        \captionsetup{labelformat=empty}
		\caption{(f) MLP$+$Dropout 0.7}
	\end{minipage}
	\begin{minipage}[t]{0.24\textwidth}  
		\centering  
		\addtocounter{figure}{-1}
		\includegraphics[width=\textwidth]{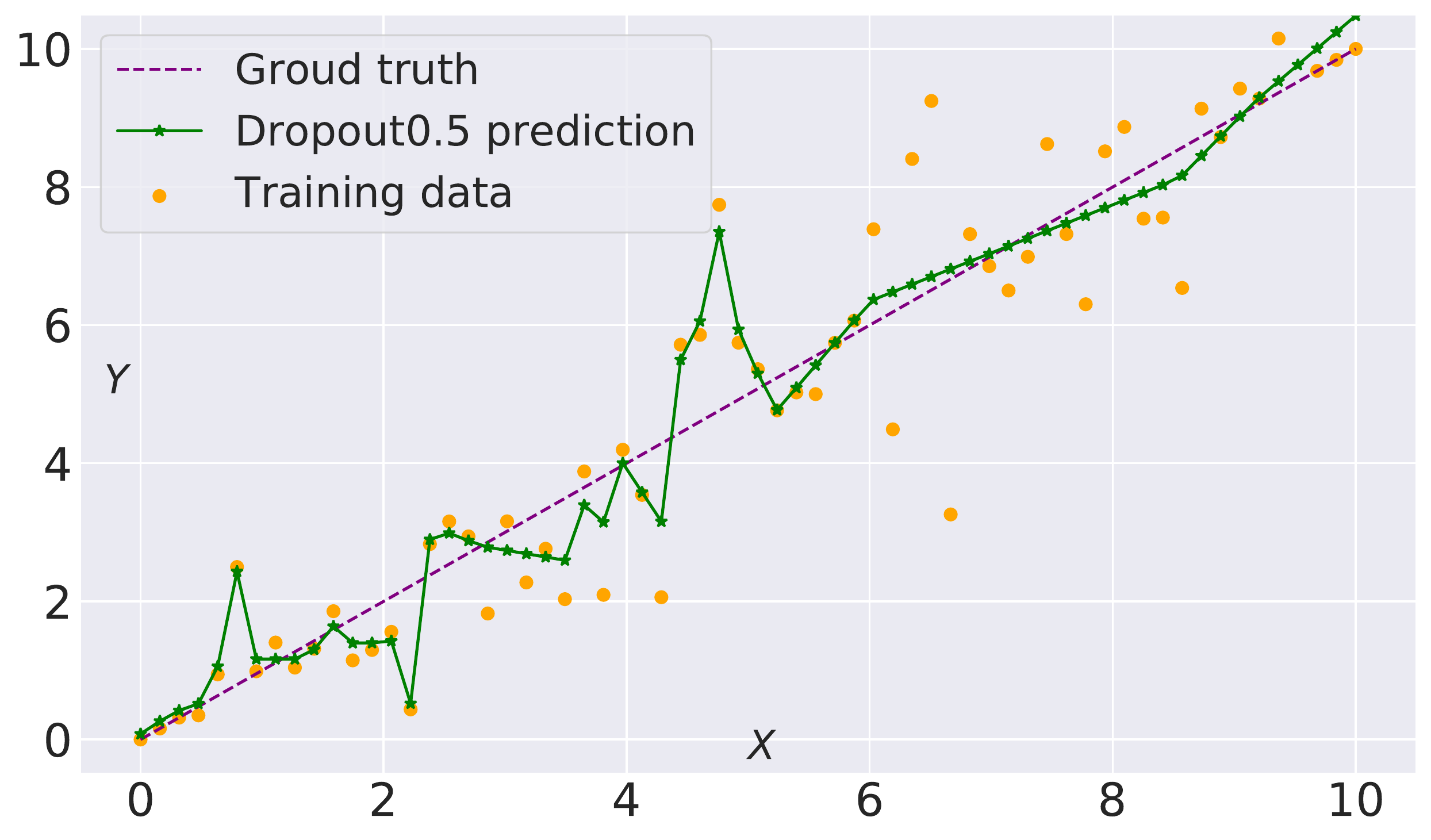}
        \captionsetup{labelformat=empty}
		\caption{(g) MLP$+$Dropout 0.5}
	\end{minipage}
	\begin{minipage}[t]{0.24\textwidth}  
		\centering  
		\addtocounter{figure}{-1}
		\includegraphics[width=\textwidth]{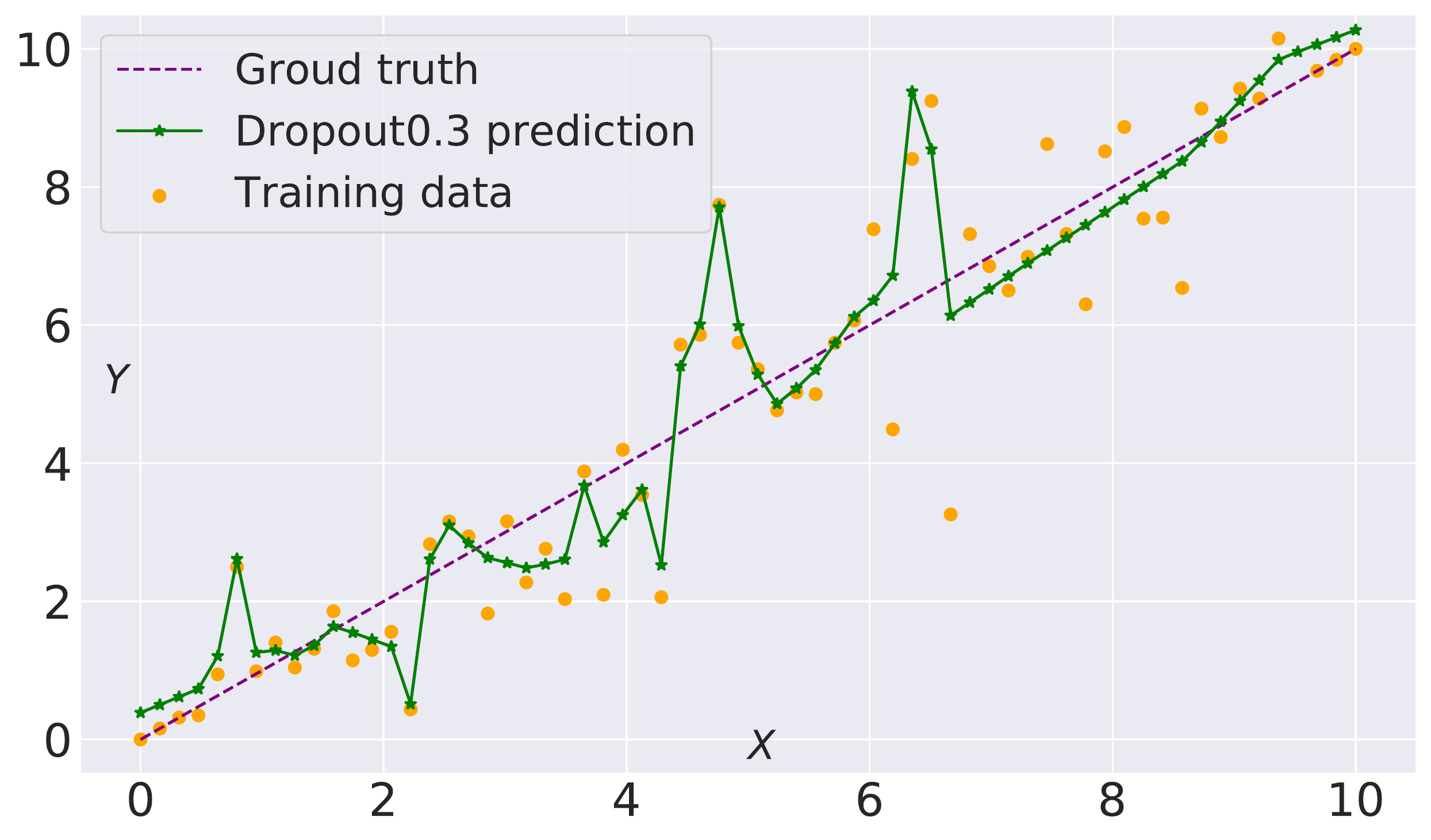}
        \captionsetup{labelformat=empty}
		\caption{(h) MLP$+$Dropout 0.3}
	\end{minipage}
	\centering  
	\caption{Comparisons of regression between standard MLP and MLP trained with Nested Dropout~\cite{rippel2014learning} and Dropout \cite{srivastava2014dropout} on a synthetic noisy label dataset. 
	(a) MLP with standard training; 
	(b-d) predictions of MLP$+$Nested using only the first $k\in\{1,10,100\}$ channels;
	(e-h) predictions of MLP$+$Dropout with drop ratio $p_{\text{drop}}\in\{0.9,0.7,0.5,0.3\}$.}
	\label{fig::ResistNoise}
	\vspace{-3mm}
\end{figure*} 

\subsection{Combination with Co-teaching} \label{subsec::comco}
In this subsection, we consider the stage two in our method where 
two networks are further fine-tuned with Co-teaching.
Recall that during the cross-update state, one network selects its small-loss instances $\mathcal{D}_1$ and send them to its peer.
Intuitively, the above process resembles the teacher and student mechanism where the teacher selects possibly clean instances for the student to learn.
In this regard, let $q_t(y|x)$ be the teacher network, $q(y|x)$ in \eqref{eq::lvm} be the student network.
The sample selection mechanism only preserves those with small loss $-\log q_t(y|x)$, i.e.,~large $q_t(y|x)$.
If we consider this selection w.r.t.~probability together with the following student network training, we reformulate the student's loss as:
\begin{align} \label{eq::co_loss}
    L_q^s (x, y) &:= q_t(y|x) L_q (x, y) 
    \nonumber\\
    & = \mathbb{E}_{q(z|x)} \big[ - q_t(y|x)\log q(y|z) \big]
\end{align}
where $q_t(y|x)$ represents $(x,y)$'s probability to be selected.
Moreover, by regarding sample selection and student network training as a whole, we redefine student network's decoder by
$q_{\text{co}}(y|x,z) \propto \exp[q_t(y|x) \log q(y|z)]$.
In order to distinguish it from the original student network's decoder $q(y|z)$, we call $q_{\text{co}}(y|x,z)$ the taught student decoder.
The following Theorem \ref{thm::co_decomp} gives the bias-variance decomposition when networks are further fine-tuned with Co-teaching.

\begin{theorem} \label{thm::co_decomp}
    Let $q_t(y|x)$ be the Co-teaching teacher network, and $q_{\text{co}}(y|x,z) = \exp[q_t(y|x) \log q(y|z)]/C_1(x,z)$ where $C_1(x,z) :=\int_{\mathcal{Y}}\exp[q_t(y|x) \log q(y|z)]\,dy \leq 1$ be the taught student decoder. 
    For the Co-teaching student loss $L_q^s$ defined in \eqref{eq::co_loss}, the risk has a bias-variance decomposition:
    \begin{align*}
        & \mathbb{E}_{p(x,y)} \big[L_q^s(X, Y)\big] = 
        \mathbb{E}_{p(x)} \Big[ \underbrace{D_{\text{KL}} \big( p(y|x) \| \tilde{q}_{\text{co}}(y|x)\big) }_{\text{bias}}  
        \\
        &  \qquad + \underbrace{\mathbb{E}_{q(z|x)} D_{\text{KL}}
        \big(\tilde{q}_{\text{co}}(y|x) \| q_{\text{co}}(y|x,z) \big)}_{\text{variance}} \Big] + \text{const}
    \end{align*}
    where $\tilde{q}_{\text{co}}(y|x) \propto \exp \big[\mathbb{E}_{q(z|x)}\log q_{\text{co}}(y|x,z) \big]$ is the average or an ensemble of models. 
    Moreover, by defining $\alpha(y|x):=\mathbb{E}_{\tilde{q}(y|x)}\big[\exp[ q_t(y|x)] \big]/ \exp[ q_t(y|x)]$, we then have:
    \begin{itemize}
        \item[\textit{(i)}] If $\alpha(y|x) \leq 1$, then
        \begin{align}
        D_{\text{KL}} \big( p(y|x) \| \tilde{q}_{\text{co}}(y|x)\big) 
        \leq D_{\text{KL}}\big( p(y|x) \| \tilde{q}(y|x) \big).
         \label{eq::ineq_bias} 
        \end{align}
        \item[\textit{(ii)}] If $\alpha(y|x) \leq C_1(x,z)$, then
        \begin{align}
        & \mathbb{E}_{q(z|x)} D_{\text{KL}}
        \big(\tilde{q}_{\text{co}}(y|x) \| q_{\text{co}}(y|x,z) \big)
        \nonumber\\
        & \quad \quad \quad
        \geq \mathbb{E}_{q(z|x)} D_{\text{KL}}\big( \tilde{q}(y|x) \| q(y|z) \big).
        \label{eq::ineq_var}
    \end{align}
    \end{itemize}
\end{theorem}

\begin{remark} \label{rk::co_bv}
    The condition $\alpha(y|x) \leq 1$ equals to $q_t(y|x) \geq \log\big[\mathbb{E}_{\tilde{q}(y|x)}\exp[ q_t(y|x)]\big]$ where the right-hand side measures the difference between a single training network $\tilde{q}(y|x)$ and the teacher network $q_t(y|x)$ of Co-teaching. 
    The larger the difference, the smaller the value, and vise versa.
    Hence, to obtain smaller bias term than in \eqref{eq::bvDecompo}, the sample selection of Co-teaching only chooses those $(x,y)$ with large $q_t(y|x)$ so as to meet the condition.
    As $C_1(x,z)\leq 1$ by definition, if further $\alpha(y|x) \leq C_1(x,z)$, i.e.,~$q_t(y|x) \geq \log\big[\mathbb{E}_{\tilde{q}(y|x)}\exp[ q_t(y|x)]/C_1(x,z)\big]$ with larger right-hand side value, then we will have larger variance term than that in \eqref{eq::bvDecompo}.
\end{remark}

Theorem \ref{thm::co_decomp} and Remark \ref{rk::co_bv} together demonstrate that choosing samples with large $q_t(y|x)$ during selection of Co-teaching leads to smaller bias term and larger variance term than those in \eqref{eq::bvDecompo}.
That is, the impact of the bias term can be even diminished.
Consequently, the sample selection mechanism during Co-teaching's cross-update process helps in further preventing networks from over-fitting on noisy labels, thus achieving better performance on clean datasets.

\section{Experiments} \label{sec::experiments}
In this section, we present our experimental results. 
First of all, we focus on how Dropout and Nested Dropout cope with the regression noise by a toy example in Section~\ref{sec:toy}. 
For better understanding of our methods, we assess them on real datasets albeit with synthetic label noise in Section \ref{subsec::synthetic}.
In Section~\ref{sec:compare}, we compare our method with the state-of-the-art methods on two real-world datasets: Clothing1M~\cite{xiao2015learning} and ANIMAL-10N~\cite{song2019selfie}. 
Finally, we conduct ablation study on ANIMAL-10N and Clothing1M in Section~\ref{sec::ablation}. 

\begin{table*}[t]
    \centering
    \renewcommand{\arraystretch}{1.1}
    \caption{
    Test accuracy (\%) of state-of-the-art methods under (a) symmetric noise on CIFAR-10 \cite{krizhevsky2009learning} and CIFAR-100 \cite{krizhevsky2009learning}, (b) 40\% asymmetric noise on CIFAR-10. 
    All approaches are implemented with PreAct ResNet-18~\cite{he2016identity} architecture.}
    \begin{minipage}{0.7\textwidth}
        \centering
        \scalebox{0.95}{\subfloat[Symmetric noise on CIFAR-10 and CIFAR-100.]{
        \begin{tabular}{ccccccccc}
            \hline\hline
            \multicolumn{1}{c|}{\multirow{2}{*}{Methods / Noise Ratio (\%)}}
            & \multicolumn{4}{c|}{CIFAR-10}              
            & \multicolumn{4}{c}{CIFAR-100}                          
            \\
            \multicolumn{1}{c|}{}   
            & \multicolumn{1}{c|}{\, 20\%\,} 
            & \multicolumn{1}{c|}{\, 50\%\,} 
            & \multicolumn{1}{c|}{\, 80\%\,} 
            & \multicolumn{1}{c|}{\, 90\%\,}         
            & \multicolumn{1}{c|}{\, 20\%\,} 
            & \multicolumn{1}{c|}{\, 50\%\,} 
            & \multicolumn{1}{c|}{\, 80\%\,} 
            & {\, 90\%\,}
            \\ \hline
            \multicolumn{1}{c|}{Cross-Entropy~\cite{li2020dividemix}} 
            & 86.8 & 79.4 & 62.9 & \multicolumn{1}{c|}{42.7} & 62.0 & 46.7 & 19.9 & 10.1
            \\
            \multicolumn{1}{c|}{Bootstrap~\cite{reed2014training}}
            & 86.8 & 79.8 & 63.3 & \multicolumn{1}{c|}{42.9} & 62.1 & 46.6 & 19.9 & 10.2
            \\
            \multicolumn{1}{c|}{F-correction~\cite{patrini2017making}}  & 86.8 & 79.8 & 63.3 & \multicolumn{1}{c|}{42.9} & 61.5 & 46.6 & 19.9 & 10.2          
            \\
            \multicolumn{1}{c|}{Co-teaching+~\cite{yu2019does,li2020dividemix}}
            & 89.5 & 85.7 & 67.4 & \multicolumn{1}{c|}{47.9} & 65.6 & 51.8 & 27.9 & 13.7         
            \\
            \multicolumn{1}{c|}{Mixup~\cite{zhang2018mixup}}  
            & 95.6 & 87.1 & 71.6 & \multicolumn{1}{c|}{52.2} & 67.8 & 57.3 & 30.8 & 14.6               
            \\
            \multicolumn{1}{c|}{PENCIL~\cite{yi2019probabilistic,li2020dividemix}}
            & 92.4 & 89.1 & 77.5 & \multicolumn{1}{c|}{58.9} & 69.4 & 57.5 & 31.1 & 15.3          
            \\
            \multicolumn{1}{c|}{MLNT~\cite{li2019learning,li2020dividemix}}  
            & 92.9 & 89.3 & 77.4 & \multicolumn{1}{c|}{58.7} & 68.5 & 59.2 & 42.4 & 19.5     
            \\
            \multicolumn{1}{c|}{M-correction~\cite{arazo2019unsupervised}}        
            & 94.0 & 92.0 & 86.8 & \multicolumn{1}{c|}{69.1} & 73.9 & 66.1 & 48.2 & 24.3       
            \\
            \multicolumn{1}{c|}{DivideMix~\cite{li2020dividemix}}       & 96.1 & 94.6 & 93.2 & \multicolumn{1}{c|}{76.0} & 77.3 & 74.6 & 60.2 & 31.5      
            \\ \hline
            \multicolumn{9}{c}{\bf Ours}                                
            \\
            \multicolumn{1}{c|}{Nested$+$Co-teaching}              
            & 95.3 & 91.9 & 78.8 & \multicolumn{1}{c|}{55.0} & 77.5 & 66.7 & 43.0 & 14.2
            \\
            \multicolumn{1}{c|}{Dropout$+$Co-teaching}              
            & 95.0 & 90.2 & 78.8 & \multicolumn{1}{c|}{56.0} & 76.7 & 67.0 & 44.1 & 18.4
            \\ \hdashline
            \multicolumn{9}{c}{\bf Pre-Cleaning with DivideMix}      
            \\
            \multicolumn{1}{c|}{M-correction} 
            & 94.5 & 93.5 & 92.6 & \multicolumn{1}{c|}{73.5} & 66.4 & 63.8 & 54.0 & 29.1
            \\
            \multicolumn{1}{c|}{Nested$+$Co-teaching}  & 96.0 & 94.9 & \bf 93.5 & \multicolumn{1}{c|}{\bf 78.3} & 78.3 & 76.0 & \bf 63.6 & 36.1
            \\
            \multicolumn{1}{c|}{Dropout$+$Co-teaching }
            & \bf 96.1 & \bf 95.0 & 93.4 & \multicolumn{1}{c|}{78.0} & \bf 79.5 & \bf 76.8 & 63.2 & \bf 37.4
            \\ 
            \hline\hline
        \end{tabular}}}
    \end{minipage}\hfill
    \begin{minipage}{0.3\textwidth}
        \centering
        \scalebox{0.95}{\subfloat[40\% asym.~noise on CIFAR-10.]{
        \begin{tabular}{c|c}
            \hline\hline
            Methods  & Acc. (\%)   
            \\ \hline
            Cross-Entropy~\cite{li2020dividemix} & 85.0 
            \\ 
            F-correction~\cite{patrini2017making,li2020dividemix} & 87.2 
            \\
            M-correction~\cite{arazo2019unsupervised,li2020dividemix} & 87.4 \\
            Iterative-CV~\cite{chen2019understanding,li2020dividemix} & 88.6 \\
            PENCIL~\cite{yi2019probabilistic,li2020dividemix} & 88.5 
            \\
            JO~\cite{tanaka2018joint,li2020dividemix} & 88.9 
            \\
            MLNT~\cite{li2019learning,li2020dividemix} & 89.2 
            \\ 
            DivideMix~\cite{li2020dividemix} & 93.4 
            \\ \hline
            \multicolumn{2}{c}{\bf Ours} 
            \\
            Nested$+$Co-teaching & 93.0 
            \\
            Dropout$+$Co-teaching & 92.9 
            \\ \hdashline
            \multicolumn{2}{c}{\bf Pre-Cleaning with DivideMix} 
            \\
            MLNT & 92.5 
            \\
            Nested$+$Co-teaching & \bf 94.2 
            \\
            Dropout$+$Co-teaching & 93.8 
            \\ \hline\hline
        \end{tabular}} 
        \label{tab::cifar}}
    \end{minipage}
    \hfill
    \vspace{-5mm}
\end{table*}

\subsection{Toy example: a simple regression with noise}
\label{sec:toy}
This subsection provides an intuitive better understanding on the reason why Nested Dropout~\cite{rippel2014learning} and Dropout~\cite{srivastava2014dropout} are able to resist label noise.
To this end, we give a simulated regression experiment. 
Specifically, we generate a dataset of noisy observations from $y_i = x_i + \epsilon_i$ for $i=1,\ldots,64$ where $x_i$ is evenly spaced between $[0,10]$ and $\epsilon_i \sim \mathcal{N}(0,1)$ are i.i.d~sampled. 
We employ a multilayer perceptron (MLP) consisting of three linear layers with input and output dimensions being $1 \rightarrow 64 \rightarrow 128 \rightarrow 1$.
Moreover, we add ReLU activations to all layers except the last one.
When training model with Nested Dropout/Dropout, we only apply it to the last layer of the MLP, and the corresponding model is denoted by MLP$+$Nested, MLP$+$Dropout, respectively. 
Note that we follow \eqref{eq::CatGaussian} where $\sigma_{\text{nest}} = 200$.
Fig.~\ref{fig::ResistNoise} gives the results after $100$k epochs.
The drop ratio $p_{\text{drop}}$ of Dropout varies in $\{0.9,0.7,0.5,0.3\}$ where the compression ratio decreases. 
As in Fig.~\ref{fig::ResistNoise}, MLP overfits the label noise while MLP$+$Nested with the first $k=1$, $k=10$ channels recover the ground-truth $y=x$ better. 
Nevertheless, MLP$+$Nested gradually overfits the label noise due to over parameterization as the number of channels increases.
As for MLP$+$Dropout, with $p_{\text{drop}}$ decreasing, the models become over-fitting the noisy labels.
However, MLP$+$Nested with $k=1$ still gives the best performance.
To conclude, both compression methods prevent networks from over-fitting the noisy patterns. 
Notably, for MLP$+$Nested, the main data structure information is contained in the first few channels, while noisy information is likely to be encoded in channels towards the end.

\subsection{Model analysis on synthetic noise} 
\label{subsec::synthetic}

\subsubsection{Datasets}
We evaluate our methods on CIFAR-10 \cite{krizhevsky2009learning} and CIFAR-100 \cite{krizhevsky2009learning} with synthetic label noise following ~\cite{tanaka2018joint,li2020dividemix,li2019learning}.
For the training data, we manually corrupt the label according to a transition matrix $Q$ with
$Q_{ij} = \mathbb{P}(y = j | y_{\text{clean}} = i)$, $i, j \in \{1,\ldots,C\}$ denoting the probability of flipping clean $ y_{\text{clean}}$ to noisy $y$.
One representative structure of the matrix $Q$ is the symmetric flipping \cite{van2015learning}, that is, $\mathrm{P}(y=i | y_{\text{clean}}=i) = 1-\tau$, $\mathrm{P}(y \neq i | y_{\text{clean}}=i) = \tau / (C-1)$ where $C$ is the number of classes and $\tau$ is called the noise ratio.
The other representative structure is the asymmetric (or pair) flipping \cite{patrini2017making} where label mistakes only happen within very similar classes, therefore should be tailored for different datasets. For example, in CIFAR-10, the asymmetric flippings follow: truck $\to$ auto-mobile, bird $\to$ airplane, deer $\to$ horse and cat $\leftrightarrow$ dog. The probability is $\tau$ for flipping from ground-truth to inaccurate class, while $1 - \tau$ for remaining uncorrupted. 

\begin{table}[t]
    \caption{Test accuracy (\%) of our Nested$+$Co-teaching and Dropout$+$Co-teaching on CIFAR-10 and CIFAR-100~\cite{krizhevsky2009learning} under symmetric noise.
    Methods based on ImageNet~\cite{deng2009imagenet} pre-trained models are marked with ``\cmark".
    All approaches are implemented with ResNet-18~\cite{he2016deep} architecture.}
    \label{tab::cifar_imgnet}
    \centering
    \resizebox{0.99\columnwidth}{!}{
    \begin{tabular}{c|c|ccc|ccc}
        \hline\hline
        \multirow{2}{*}{
        \begin{tabular}[c]{@{}c@{}}{Methods /} \\ Noise Ratio\end{tabular}} 
        & \multirow{2}{*}{
        \begin{tabular}[c]{@{}c@{}}{Pre-} \\ trained\end{tabular}} 
        & \multicolumn{3}{c|}{CIFAR-10} & \multicolumn{3}{c}{CIFAR-100} 
        \\
        & & \multicolumn{1}{c|}{20\%} 
        & \multicolumn{1}{c|}{50\%} 
        & 80\% 
        & \multicolumn{1}{c|}{20\%} 
        & \multicolumn{1}{c|}{50\%} & 80\%
        \\ \hline
        \multirow{2}{*}{\begin{tabular}[c]{@{}c@{}}{Nested$+$}\\ {Co-teaching}\end{tabular}}   
        & \xmark & 91.7 & 86.9 & 53.4 & 69.0 & 60.6 & 28.0
        \\
        & \cmark & \bf 92.9 & 87.8 & 72.2 & 71.4 & 63.3 & 36.4 
        \\ \hline
        \multirow{2}{*}{\begin{tabular}[c]{@{}c@{}}{Dropout$+$}\\ {Co-teaching}\end{tabular}}  
        & \xmark & 91.9 & 86.6 & 59.0 & 72.4 & 62.4 & 33.2 
        \\
        & \cmark & 91.9 & \bf 89.7 & \bf 77.9 & \bf 74.2 & \bf 65.0 & \bf 39.8
        \\ \hline\hline
    \end{tabular}}
    \vspace{-5mm}
\end{table}

\subsubsection{Implementation details}
Our methods are implemented with PyTorch.
Following the previous works~\cite{li2020dividemix,arazo2019unsupervised}, experiments on CIFAR-10, CIFAR-100 are with PreAct ResNet-18 \cite{he2016identity} trained from scratch.
In the first stage, we use SGD optimizer with a momentum of 0.9, a weight decay of 1e-4, an initial learning rate of 0.1, and batch size of 128.
We apply learning rate warm-up with $6000$ iterations and the number of epochs is 200 with learning rate decayed by 0.1 at 100 and 150 epochs.
Mixup data augmentation~\cite{zhang2018mixup} is adopted in the first stage for better performance as in~\cite{li2020dividemix,arazo2019unsupervised}.
We apply Dropout/Nested Dropout on the average pooled \textit{Conv5} features.
In the second stage, two well-trained models are set as base models for Co-teaching. 
The initial learning rate is 1e-3 and we still employ SGD as optimizer.
Moreover, $\lambda_{\text{forget}}$ is tuned under different noise ratio, batch norm is frozen and no warm-up is applied.
Models are trained for 100 epochs with the learning rate decayed by 0.1 after 50 epochs.
We set $\sigma_{\text{nest}}=50$, $p_{\text{drop}}=0.5$ for cases on CIFAR-10, and $\sigma_{\text{nest}}=100$, $p_{\text{drop}}=0.3$ on CIFAR-100.
Note that when training with Nested Dropout, we record the optimal number of channels $k^*$ of the model and use only these first $k^*$ channels when testing.

We further show that our methods can also serve as good complementary strategies to other state-of-the-art methods to achieve even better performance.
In particular, we propose an additional data preprocessing step before training our methods, which is named ``Pre-Cleaning".
During this pre-cleaning, for example, we substitute the original labels of the dataset with the predictions of a well-trained DivideMix~\cite{li2020dividemix} model, resulting in a pre-cleaned dataset.
We later train our methods on this pre-cleaned dataset so as to further exceed the performance of DivideMix.
Note that we can employ any state-of-the-art methods to conduct the ``Pre-Cleaning".

\subsubsection{Results on CIFAR-10 and CIFAR-100~\cite{krizhevsky2009learning}}
We compare our two-stage methods with multiple state-of-the-art methods on CIFAR-10 and CIFAR-100 under different types and levels of synthetic label noise in Table \ref{tab::cifar}.
We consider the performance of our methods with and without the ``Pre-Cleaning" step separately.
Without the pre-cleaning step, our simple Nested$+$Co-teaching and Dropout$+$Co-teaching achieve the top-3 performance for all except the extreme 90\% label noise ratio cases.
Moreover, by pre-cleaning with DivideMix, our methods achieve the best performance for all cases.
Note that we also compare with M-correction~\cite{arazo2019unsupervised} and MLNT~\cite{li2019learning} with the pre-cleaning step in Table~\ref{tab::cifar}.
It can be seen that although both M-correction and MLNT improve upon their own results with the help of pre-cleaning, they fail to surpass the performance of DivideMix.
Therefore, they cannot serve as complementary strategies to DivideMix to enhance performance.
In contrast, our Dropout$+$Co-teaching improves upon DivideMix by a maximum of 5.9\% in accuracy under 90\% symmetric noise on CIFAR-100.
We also consider the training and inference time of our methods. 
It takes 2.7 hours for training the complete two-stage model on a single NVIDIA V100 GPU, and both Nested$+$Co-teaching and Dropout$+$Co-teaching take 0.24 milliseconds per image for inference.
In regard of above, our methods not only perform well on their own, but can also serve as effective and efficient complementary strategies to other state-of-the-art methods with only a little extra time.
Moreover, we provide additional insight that by using ImageNet~\cite{deng2009imagenet} pre-trained models as in Table~\ref{tab::cifar_imgnet}, we can achieve even better performance.

\subsection{Comparison with state-of-the-art methods on real datasets}
\label{sec:compare}

\subsubsection{Datasets}
The following experiments are conducted on two real-world datasets with real label noise: Clothing1M~\cite{xiao2015learning} and ANIMAL-10N~\cite{song2019selfie}. 
Clothing1M is a benchmark dataset containing $1$ million clothing images with $14$ categories from online shopping websites, and its overall estimated noise ratio is 38.5\% according to~\cite{xiao2015learning}.
Moreover, this dataset provides $50$k, $14$k and $10$k manually verified clean data for training, validation and testing. 
Note that we do not use the clean training set during training.
In our experiment, we randomly sample a balanced subset that includes $260$k images with $18.5$k images per category, from the noisy training set as in \cite{yi2019probabilistic,zhang2021learning}.
This balanced subset is used as our training set and classification accuracies are reported on the $10$k clean test data. 
We follow the data augmentations in ~\cite{li2020dividemix,liu2020early,arazo2019unsupervised}, which includes Mixup data augmentation~\cite{zhang2018mixup}.
ANIMAL-10N is another benchmark dataset recently proposed by~\cite{song2019selfie}. 
It contains $10$ animal classes with confusing appearance.
There are $50$k training, $5$k testing images, and an estimated label noise rate $8\%$. 
No data augmentation is applied so as to follow the settings in \cite{song2019selfie}.

\begin{table}[t]
    \caption{Test accuracy (\%) of state-of-the-art methods on Clothing1M~\cite{xiao2015learning}. 
    All approaches are implemented with ResNet-50~\cite{he2016deep} architecture.
    Results with ``*" use a balanced subset or a balanced loss.}
    \vspace{-2mm}
    \label{tab::clothing1m}
    \begin{center}
    \begin{tabular}{c|c}
        \hline \hline
        Methods  & {\,Acc. (\%)\,}   
        \\ \hline
        Cross-Entropy~\cite{li2020dividemix} & 69.2 
        \\ 
        F-correction~\cite{patrini2017making} & 69.8 
        \\
        M-correction~\cite{arazo2019unsupervised} & 71.0 
        \\
        JO~\cite{tanaka2018joint} & 72.2 
        \\
        ELR*~\cite{liu2020early} & 72.9 
        \\
        HOC*~\cite{zhu2021clusterability} & 73.4
        \\
        PENCIL*~\cite{yi2019probabilistic} & 73.5
        \\
        MLNT~\cite{li2019learning} & 73.5 
        \\ 
        PLC*~\cite{zhang2021learning} & 74.0
        \\
        C2D*~\cite{zheltonozhskii2022contrast} & 74.6
        \\
        ELR+*~\cite{liu2020early} & 74.8 
        \\
        DivideMix*~\cite{li2020dividemix} & 74.8 
        \\ \hline
        \multicolumn{2}{c}{\bf Ours} 
        \\
        Nested* & 73.2 
        \\
        Nested$+$Co-teaching* & \bf 75.0 
        \\
        Dropout* & 72.9
        \\
        Dropout$+$Co-teaching* & 74.0
        \\ \hline\hline
    \end{tabular}
    \end{center}
    \vspace{-5mm}
\end{table}

\subsubsection{Implementation details}
The following experiments are implemented on PyTorch. 
Experiments on Clothing1M~\cite{xiao2015learning} are with ResNet-50~\cite{he2016deep} pre-trained on ImageNet~\cite{deng2009imagenet} following \cite{yi2019probabilistic,li2020dividemix,li2019learning}. 
Dropout/Nested Dropout is applied right before the linear classifier in the network with $\sigma_{\text{nest}}=250$, $p_{\text{drop}}=0.5$.
First, SGD optimizer is used for stage one model training with momentum $0.9$, weight decay 5e-4, initial learning rate 2e-2, and batch size $96$. 
Learning rate warm-up is utilized for $6000$ iterations in stage one, and model is later trained for $30$ epochs with the learning rate decayed by $0.1$ after the $10$-th epoch.
Second, the two models trained in stage are fine-tuned through Co-teaching in stage two.
SGD optimizer is utilized with the same settings and follows a cosine learning rate decay~\cite{loshchilov2016sgdr} with a maximum learning rate
of 5e-5 and a minimum of 1e-5 and without learning rate warm-up. 
Forget rate $\lambda_{\text{forget}}$ is set to be 0.5, batch norms are frozen, and we train for 10 epochs.
Note that when training with Nested Dropout, 
the optimal number of channels $k^*$ of the model is recorded, and only these first $k^*$ channels are used for testing.

For ANIMAL-10N, VGG-19~\cite{simonyan2014very} is used with batch normalization \cite{ioffe2015batch} as in \cite{song2019selfie}. 
The two Dropout layers in the original VGG-19 architecture are substituted with Nested Dropout when ``Nested" is applied.
SGD optimizer is applied.
For more stable training, we use alternative training strategy for these two layers of Dropout/Nested Dropout.
That is, for each feed-forward, Nested Dropout is either applied to the first or the second layers.
In stage one, following \cite{song2019selfie}, the network is trained for $100$ epochs with initial learning rate $0.1$.
The learning rate is later decayed by $0.2$ at $50$-th and $75$-th epochs.
Moreover,  models are trained with learning rate warm-up for $6000$ iterations. 
In stage two, forget rate $\lambda_{\text{forget}}$ is set to be $0.2$, batch norms are frozen and no warm-up is applied.
The initial learning rate is $4e{-3}$ and is decayed by $0.2$ after the $5$-th epoch with $30$ epochs in total.

\subsubsection{Results on the Clothing1M~\cite{xiao2015learning}} 
We compare our methods to state-of-the-art methods in Table~\ref{tab::clothing1m}. 
Notably, a single model trained with Nested Dropout or Dropout can not only surpass M-correction~\cite{arazo2019unsupervised}, JO~\cite{tanaka2018joint}, ELR~\cite{liu2020early}, but also
achieve comparable performance to HOC~\cite{zhu2021clusterability} and PENCIL~\cite{yi2019probabilistic}. 
The combination of Dropout$+$Co-teaching boosts the performance of a single Dropout model by $1.1\%$.
Moreover, the combination of Nested$+$Co-teaching boosts from $73.2\%$ of a single model to $75.0\%$, achieving the best among all methods.

\subsubsection{Results on the ANIMAL-10N~\cite{song2019selfie}} 
Table~\ref{tab::animal} gives the results on ANIMAL-10N. 
It can be seen that our single Dropout model can achieve comparable performance to SELFIE~\cite{song2019selfie}.
Moreover, the combination with Co-teaching provides a consistent performance boost, which is in line with the results on Clothing1M~\cite{xiao2015learning}. 
Notably, our best performance by using Dropout$+$Co-teaching achieves $84.5\%$ accuracy outperforms recent approach PLC~\cite{zhang2021learning} by $1.1\%$.

\subsection{Ablation study with real label noise}
\label{sec::ablation}
This section provides ablation study of $\sigma_{\text{nest}}$, $p_{\text{drop}}$, and $\lambda_{\text{forget}}$ on ANIMAL-10N~\cite{song2019selfie}. 
Note that same as many state-of-the-art methods~\cite{song2019selfie,li2020dividemix,liu2020early,li2019learning}, our hyper-parameters need to be tuned on a clean validation set.
Moreover, we also evaluate our methods using different backbones on Clothing1M~\cite{xiao2015learning}.

\subsubsection{Ablation on $\sigma_{\text{nest}}$}
As in Table~\ref{tab::ablation} (a), Nested Dropout provides consistent improvement compared to training with standard cross-entropy loss and the performance gain is also robust to the choices of the hyper-parameter $\sigma_{\text{nest}}$. 
Moreover, fine-tuning through Co-teaching provides clear performance boost for all the models. 
We also present the optimal number of channels of each model (entry ``$k^*$"). 
Although there are two layers of Nested Dropout applied to the classifier of VGG-19, the optimal number of channels $k^*$ is recorded w.r.t.~the last Nested Dropout layer for simplicity.
Interestingly, models trained with Nested Dropout achieve better performance with only less than $1$\% of channels comparing to their counterparts with cross-entropy loss.

\begin{table}[t]
    \caption{Average test accuracy (\%) with standard deviation (3 runs) of state-of-the-art methods on ANIMAL-10N~\cite{song2019selfie}. All approaches are implemented with VGG-19~\cite{simonyan2014very} architecture.
    Results with ``*" use two networks for training.}
    \vspace{-2mm}
    \label{tab::animal}
    \begin{center}
    \begin{tabular}{c|c}
        \hline\hline
        Methods  & Acc. (\%)   
        \\ \hline
        Cross-Entropy~\cite{song2019selfie} & 79.4 {$\pm$} 0.1
        \\
        Co-teaching*~\cite{han2018co,song2019selfie} & 80.2 {$\pm$} 0.1
        \\
        ActiveBias~\cite{chang2017active,song2019selfie} & 80.5 $\pm$ 0.3
        \\
        SELFIE~\cite{song2019selfie} & 81.8 {$\pm$} 0.1
        \\
        PLC~\cite{zhang2021learning} & 83.4 {$\pm$} 0.4
        \\ \hline
        \multicolumn{2}{c}{\bf Ours} 
        \\
        Nested & 81.3 {$\pm$} 0.6
        \\
        Nested$+$Co-teaching* & 84.1 {$\pm$} 0.1
        \\
        Dropout & 81.6 {$\pm$} 0.2
        \\
        Dropout$+$Co-teaching* & \bf 84.5 {$\pm$} 0.1
        \\ \hline\hline
    \end{tabular}
    \end{center}
    \vspace{-5mm}
\end{table}

\subsubsection{Ablation on $p_{\text{drop}}$}
We experiment on different $p_{\text{drop}}$ with results given in Table \ref{tab::ablation} (a).
Note that results with $p_{\text{drop}}\geq 0.5$ are not given in the table since the training of single VGG-19 fails on ANIMAL-10N \cite{song2019selfie}.
The performance under different choices of $p_{\text{drop}}$ are less robust compared to those of Nested Dropout.
However, what is in common is that the combination with Co-teaching again brings significant performance boost for all the models.

\begin{table*}[t]
    \centering
    \renewcommand{\arraystretch}{1.1}
    \caption{Average test accuracy (\%) with standard deviation (3 runs) of 
    (a) different $\sigma_{\text{nest}}$ for Nested Dropout, $p_{\text{drop}}$ for Dropout. The corresponding optimal number of channels $k^*$ for each Nested Dropout model is given (entry ``$k^*$"). We report test accuracy of single model (entry ``Acc.") as well as the accuracy with the combination of Co-teaching (entry ``Co-teaching Acc.")
    (b) different $\lambda_{\text{forget}}$ for Co-teaching on ANIMAL-10N~\cite{song2019selfie}.
    }
    \begin{minipage}{0.55\textwidth}
        \centering
        \scalebox{0.9}{\subfloat[Ablation study on $\sigma_{\text{nest}}$ and $p_{\text{drop}}$.]{
        \begin{tabular}{cccccc}
            \hline\hline
            \multicolumn{2}{c|}{Methods} & $k^*$  
            & \multicolumn{1}{c|}{Acc. (\%)}  & $k^*$  & Co-teaching Acc. (\%)  \\ \hline
            \multicolumn{2}{c|}{Cross-Entropy} & 4096  
            & \multicolumn{1}{c|}{79.4 $\pm$ 0.1} & 4096 & 82.2 $\pm$ 1.1 
            \\ \hline
            \multicolumn{1}{c|}{\multirow{5}{*}{$\sigma_{\text{nest}}$}} 
            & \multicolumn{1}{c|}{25} & 17.7 $\pm$ 9.7 
            & \multicolumn{1}{c|}{81.0 $\pm$ 0.6} & 16.3 $\pm$ 6.9 & 83.7 $\pm$ 0.1 
            \\ 
            \multicolumn{1}{c|}{} & \multicolumn{1}{c|}{50}  & 18.8 $\pm$ 6.9 & \multicolumn{1}{c|}{\bf 81.3 $\pm$ 0.6} & 13.4 $\pm$ 4.1   & 84.1 $\pm$ 0.2 
            \\ 
            \multicolumn{1}{c|}{} & \multicolumn{1}{c|}{100} & 13.6 $\pm$ 5.6 & \multicolumn{1}{c|}{81.0 $\pm$ 0.5} & 16.8 $\pm$ 7.1   & \bf 84.1 $\pm$ 0.1 
            \\ 
            \multicolumn{1}{c|}{} & \multicolumn{1}{c|}{150} & 16.0 $\pm$ 3.6 & \multicolumn{1}{c|}{81.1 $\pm$ 0.5} & 18.8 $\pm$ 7.4   & 83.8 $\pm$ 0.2 
            \\ 
            \multicolumn{1}{c|}{} & \multicolumn{1}{c|}{250} & 13.2 $\pm$ 3.1 & \multicolumn{1}{c|}{81.1 $\pm$ 0.2} & 21.0 $\pm$ 10.4  & 83.8 $\pm$ 0.1 
            \\ \hline
            &   & \multicolumn{2}{c}{Acc. (\%)}                        
            & \multicolumn{2}{c}{Co-teaching Acc. (\%)} 
            \\ 
            \multicolumn{1}{c|}{\multirow{3}{*}{$p_{\text{drop}}$}} 
            & \multicolumn{1}{c|}{0.1} 
            & \multicolumn{2}{c|}{\bf 81.6 $\pm$ 0.2} 
            & \multicolumn{2}{c}{\bf 84.5 $\pm$ 0.1} 
            \\ 
            \multicolumn{1}{c|}{} & \multicolumn{1}{c|}{0.3} & \multicolumn{2}{c|}{80.8 $\pm$ 0.4}
            & \multicolumn{2}{c}{84.0 $\pm$ 0.2} 
            \\ 
            \multicolumn{1}{c|}{} & \multicolumn{1}{c|}{0.5} & \multicolumn{2}{c|}{81.1 $\pm$ 0.8} & \multicolumn{2}{c}{84.4 $\pm$ 0.2} 
            \\ \hline\hline
        \end{tabular}} 
        \label{tab::animal10}}
    \end{minipage}
    \hfill
    \begin{minipage}{0.44\textwidth}
        \centering
        \scalebox{0.9}{\subfloat[Ablation study on $\lambda_{\text{forget}}$.]{
        \begin{tabular}{c|cccc}
            \hline\hline
            $\lambda_{\text{forget}}$  
            & 0.1 & 0.2 & 0.3 & 0.5 
            \\ \hline
            \multirow{4}{*}{Acc (\%)} 
            & \multicolumn{4}{c}{Nested$+$Co-teaching} 
            \\  
            & 84.1 {$\pm$} 0.1 & \bf 84.1 {$\pm$} 0.1 & 83.3 {$\pm$} 0.2 & 83.3 {$\pm$} 0.2 
            \\ \cline{2-5} 
            & \multicolumn{4}{c}{Dropout$+$Co-teaching} 
            \\  
            & 84.4 {$\pm$} 0.1 & \bf 84.5 {$\pm$} 0.1 & 84.0 {$\pm$} 0.1 & 83.4 {$\pm$} 0.3  
            \\ \hline\hline
        \end{tabular}} 
        \label{tab::ablation}}
    \end{minipage}\hfill
    \vspace{-5mm}
\end{table*}

\begin{table}[t]
    \caption{Test accuracy (\%) on Clothing1M~\cite{xiao2015learning} with different backbones: ResNet-18, ResNet-50, EfficientNet-B2~\cite{tan2019efficientnet}.
    Results with ``*" use a balanced subset or a balanced loss.}
    \label{tab::clothing1m_backbone}  
    \centering
    \resizebox{0.95\columnwidth}{!}{
    \begin{tabular}{cccc}
        \hline\hline
        \multicolumn{1}{c|}{\multirow{2}{*}{Method / Acc. (\%)}} 
        & \multicolumn{3}{c}{Backbones}           
        \\
        \multicolumn{1}{c|}{} 
        & ResNet-18 & ResNet-50 & EfficientNet-B2 
        \\ \hline
        \multicolumn{1}{c|}{Cross-Entropy}                       
        & 67.2~\cite{wei2020combating} & 69.2~\cite{li2020dividemix} & 69.8 
        \\
        \multicolumn{1}{c|}{DivideMix*~\cite{li2020dividemix}}  
        & --        & 74.8      & --              
        \\ \hline
        \multicolumn{4}{c}{\bf Ours} 
        \\
        \multicolumn{1}{c|}{Nested*}                             
        & 73.1      & 73.2      & 72.5            
        \\
        \multicolumn{1}{c|}{Nested+Co-teaching*}                 
        & \bf 74.9  & \bf 75.0  & \bf 73.5            
        \\
        \multicolumn{1}{c|}{Dropout*}                             
        & 72.5      & 72.9      & 72.5            
        \\
        \multicolumn{1}{c|}{Dropout+Co-teaching*}                
        & 74.0      & 74.0      & 73.4            
        \\ \hline\hline
    \end{tabular}}
\end{table}

\subsubsection{Ablation on $\lambda_{\text{forget}}$}
We focus on how forget rate $\lambda_{\text{forget}}$ of Co-teaching influences the performance in Table~\ref{tab::ablation} (b).
We present with the settings where Nested Dropout and Dropout have the best performance.
To be specific, we set $\sigma_{\text{nest}}=100$ and $p_{\text{drop}}=0.1$ for the training, respectively.
For Nested$+$Co-teaching, the performance drops with $\lambda_{\text{forget}}$ increasing, and the accuracy remains $83.3\%$ for $\lambda_{\text{forget}}\geq 0.5$.
The performance with $\lambda_{\text{forget}}\geq 0.5$ are not given in the Table~\ref{tab::ablation} (b) for simplicity.
Moreover, this performance boost of $83.3\%$ compared to a single Nested Dropout network actually results from the ensemble estimation, not the cross-update mechanism of Co-teaching.
Therefore, considering that the estimated label noise ratio of ANIMAL-10N is $8\%$, Co-teaching's core mechanism is not suitable for cases where the difference between $\lambda_{\text{forget}}$ and the ground-truth noise ratio is large.
Furthermore, performance of Dropout$+$Co-teaching again verify the above analysis where the best performance is achieved by $\lambda_{\text{forget}}=0.2$.

\subsubsection{Ablation on different backbones}
As given in Table~\ref{tab::clothing1m_backbone}, the best performance is achieved by using ResNet-50 as backbone, which is 75.0\% in accuracy.
Even though EfficientNet-B2~\cite{tan2019efficientnet} cannot achieve the best performance possibly due to its capacity issue, Nested$+$Co-teaching indeed improves upon a single Nested Dropout model.
The same applied to Dropout$+$Co-teaching.
These results suggest the effectiveness of our approach with different backbones.

\section{Conclusion} 
\label{sec::conclusion}
In this paper, we investigate the problem of image classification in the presence of label noise.
In particular, we find that preventing networks from over-fitting the corrupted labels is one key problem in learning with noisy labels based on a bias-variance decomposition.
To this end, we introduce compression inductive bias to networks to increase the variance term so as to weaken the influence of the bias term which is associated with over-fitting.
This inductive bias is realized by applying simple compression regularizations such as Dropout \cite{srivastava2014dropout} and its variant named Nested Dropout \cite{rippel2014learning} to networks.
Notably, Nested Dropout is proved to learn ordered feature representations in this paper.
Therefore, this information sorting property can bring interpretability w.r.t.~channel importance to networks while filter out the noisy patterns.
Moreover, we combine these compression regularizations with Co-teaching \cite{han2018co}, leading to a two-stage method.
We then theoretically verify that this combination is in line with our bias-variance trade-off since Co-teaching further increases the variance term, hence further prevent networks from over-fitting.
Our method is validated on benchmark real datasets under synthetic label noise and real-world label noise including Clothing1M \cite{xiao2015learning} and ANIMAL-10N \cite{song2019selfie}.
Our method achieves comparable or even better performance than state-of-the-art approaches.
Our approach is simple compared to many existing methods.  
Therefore, we hope that our approach can serve as a strong baseline for future research on learning with noisy labels.

\section*{Acknowledgements}
We appreciate Qinghua Tao for helpful discussions, and also the anonymous reviewers for their insightful comments.

\bibliographystyle{IEEEtran}
\bibliography{egbib}

\section*{Appendix} \label{sec::app}

This appendix provides the proofs of this paper.
\begin{proof} [Proof of Theorem \ref{thm::decomp}]
Since $H(Y|X, \varepsilon)$ is a constant with respect to the model distribution $q(y|x)$, for a given input $x$, we have
\begin{align} \label{eq::decomp}
    & \mathbb{E}_{p(y|x)}\big[ L_q(x,Y) \big]
    = \mathbb{E}_{p(y|x)} \mathbb{E}_{q(z|x)} \big[ -\log q(y|z)\big]
    \nonumber\\
    & = \mathbb{E}_{p(\varepsilon)} \mathbb{E}_{p(y|x, \varepsilon)}
    \mathbb{E}_{q(z|x)} \big[ -\log q(y|z)\big]
    \nonumber\\
    & = \mathbb{E}_{p(\varepsilon)} \mathbb{E}_{p(y|x, \varepsilon)}
    \mathbb{E}_{q(z|x)} \bigg[ \log \frac{p(y|x, \varepsilon)}{q(y|z)}
    - \log p(y|x, \varepsilon) \bigg]
    \nonumber\\
    & = \mathbb{E}_{p(\varepsilon)} \mathbb{E}_{q(z|x)} 
    D_{\text{KL}} \big(p(y|x, \varepsilon) \| q(y|z) \big)
    + H(Y|X=x, \varepsilon).
\end{align}
The following result from Theorem 3.1 in \cite{Brofos2019ABD} is in need.
\begin{lemma} \label{lm::decomp}
Let $\theta \sim \pi$ and $\omega \sim \pi'$ and let $f_{\theta}$ be a distribution with the same supports as a random probability distribution $P_{\omega}$ for all $\theta$ and $\omega$. 
Then if $Y\sim \bar{P}:=\mathbb{E}_{\omega} P_{\omega}$,
\begin{align*}
    \mathbb{E}_{\omega}\mathbb{E}_{\theta} D_{\text{KL}} (P_{\omega} \| f_{\theta})
    = D_{\text{KL}}(\bar{P} \| \tilde{P}) + \mathbb{E}_{\theta} D_{\text{KL}}(\tilde{P} \| f_{\theta}) + I(Y; \omega)
\end{align*}
where $\tilde{P}(\cdot) \propto \exp \big[ \mathbb{E}_{\theta} \log f_{\theta} (\cdot)\big]$.
\end{lemma}
Now by applying Lemma \ref{lm::decomp}, we find that the first term in \eqref{eq::decomp} can be decomposed as
\begin{align*}
    \mathbb{E}_{p(\varepsilon)} \mathbb{E}_{q(z|x)} 
    & D_{\text{KL}} \big(p(y|x, \varepsilon) \| q(y|z) \big)
    = D_{\text{KL}} \big(p(y|x) \| \tilde{q}(y|x) \big) 
    \\
    & + \mathbb{E}_{q(z|x)} D_{\text{KL}} \big(\tilde{q}(y|x) \| q(y|z) \big)
    + I(Y; \varepsilon)
\end{align*}
where we have $p(y|x) = \mathbb{E}_{p(\varepsilon)}p(y|x,\varepsilon)$, and also $\tilde{q}(y|x) = \exp [ \mathbb{E}_{q(z|x)} \log q(y|z)]/\int_{\mathcal{Y}} \exp [ \mathbb{E}_{q(z|x)} \log q(y|z)] dy$. 
Note the proof follows if we plug in the above decomposition when computing $\mathbb{E}_{p(x,y)}\big[ L_q(X,Y)\big]$.
Note that both $H(Y|X,\varepsilon)$ and $I(Y;\varepsilon)$ are constant with respect to the model $q(y|x)$.
This completes the proof.
\end{proof}

\begin{proof} [Proof of Theorem \ref{thm::ranking}]
We split the proof into two parts.

    \textbf{\textit{(i)}} [Proof of \eqref{eq::IYH}]:
    For the hidden representation $\tilde{Z}$, since we have \eqref{eq::permu} held, $\forall\, i,j \in \{1,\ldots,K\}$, and $i<j$ without loss of generality, we exchange the $i$-th and $j$-th arguments:
    \begin{align*}
        p_{\tilde{Z}_i}(\tilde{z}_i) 
        & =\int\cdots\int p_{\tilde{Z}_1,\tilde{Z}_2,\ldots,\tilde{Z}_K}
        (\tilde{z}_1,\ldots,\tilde{z}_i,\ldots,\tilde{z}_j,\ldots,\tilde{z}_K) 
        \\ 
        & \quad
        d\tilde{z}_1\cdots d\tilde{z}_{i-1}d\tilde{z}_{i+1}\cdots d\tilde{z}_K
        \\
        &=\int\cdots\int p_{\tilde{Z}_1,\tilde{Z}_2,\ldots,\tilde{Z}_K}
        (\tilde{z}_1,\ldots,\tilde{z}_j,\ldots,\tilde{z}_i,\ldots,\tilde{z}_K) 
        \\
        & \quad
        d\tilde{z}_1\cdots d\tilde{z}_{i-1}d\tilde{z}_{i+1}\cdots d\tilde{z}_K
        =p_{\tilde{Z}_j}(\tilde{z}_i).
    \end{align*}
    Therefore, $p_{\tilde{Z}_i}(\tilde{z}_i) = p_{\tilde{Z}_j}(\tilde{z}_i)$ and $\tilde{Z}_i$ and $\tilde{Z}_j$ are identically distributed.
    Since $i$, $j$ are arbitrarily chosen, we have $\{\tilde{Z}_i\}_{i=1}^K$ identically distributed.
    
    Similarly, for the joint distributions of $Y, \tilde{Z}_i$ and $Y, \tilde{Z}_j$ with $i$, $j$ arbitrarily chosen, we have
    \vspace{-1mm}
    \begin{align*}
    &p_{Y,\tilde{Z}_i}(y,\tilde{z}_i) 
        \\
        &= \int\cdots\int p_{Y,\tilde{Z}_1,\tilde{Z}_2,\ldots,\tilde{Z}_K}(y,\tilde{z}_1,\ldots, \tilde{z}_i,\ldots,\tilde{z}_j,\ldots,\tilde{z}_K)
        \\
        & \quad
        d\tilde{z}_1\cdots d\tilde{z}_{i-1}d\tilde{z}_{i+1}\cdots d\tilde{z}_K
        \\
        &= \int\cdots\int p_{Y|\tilde{Z}_1,\tilde{Z}_2,\ldots,\tilde{Z}_K}
        (y|\ldots,\tilde{z}_i,\ldots,\tilde{z}_j,\ldots) 
        \\
        & p_{\tilde{Z}_1,\tilde{Z}_2,\ldots,\tilde{Z}_K}
        (\ldots,\tilde{z}_i,\ldots,\tilde{z}_j,\ldots) 
        d\tilde{z}_1\cdots d\tilde{z}_{i-1}d\tilde{z}_{i+1}\cdots d\tilde{z}_K
        \\
        &= \int\cdots\int p_{Y|\tilde{Z}_1,\tilde{Z}_2,\ldots,\tilde{Z}_K}
        (y|\ldots,\tilde{z}_j,\ldots,\tilde{z}_i,\ldots)
        \\
        & p_{\tilde{Z}_1,\tilde{Z}_2,\ldots,\tilde{Z}_K}
        (\ldots,\tilde{z}_j,\ldots,\tilde{z}_i,\ldots) 
        d\tilde{z}_1\cdots d\tilde{z}_{i-1}d\tilde{z}_{i+1}\cdots d\tilde{z}_K
        \\
        &= \int\cdots\int p_{Y,\tilde{Z}_1,\tilde{Z}_2,\ldots,\tilde{Z}_K}(y,\tilde{z}_1,\ldots, \tilde{z}_j,\ldots,\tilde{z}_i,\ldots,\tilde{z}_K)
        \\
        & \quad
        d\tilde{z}_1\cdots d\tilde{z}_{i-1}d\tilde{z}_{i+1}\cdots d\tilde{z}_K
        = p_{Y,\tilde{Z}_j}(y,\tilde{z}_i),
    \end{align*}
    where the third equivalence holds for the permutation invariant \eqref{eq::permu},\eqref{eq::dec_permu}.
    Hence, $p_{Y,\tilde{Z}_i}(y,\tilde{z}_i) = p_{Y,\tilde{Z}_j}(y,\tilde{z}_i)$ for arbitrary $i$, $j$, that is, $\{(Y,\tilde{Z}_i)\}_{i=1}^K$ are identically distributed.

    Considering the formulation of mutual information, we have $\forall\, i,j \in \{1,\ldots,K\}$:
    \begin{align*}
        I(Y; \tilde{Z}_i) :&= \int_{\mathcal{Y}} \int_{\mathcal{\tilde{Z}}_i} p_{Y,\tilde{Z}_i}(y,\tilde{z}_i) \log \frac{p_{Y,\tilde{Z}_i}(y,\tilde{z}_i)}{p_{Y}(y) p_{\tilde{Z}_i}(\tilde{z}_i)} \, dy d\tilde{z}_i
        \\
        & = \int_{\mathcal{Y}} \int_{\mathcal{\tilde{Z}}_j} p_{Y,\tilde{Z}_j}(y,\tilde{z}_j) \log \frac{p_{Y,\tilde{Z}_j}(y,\tilde{z}_j)}{p_{Y}(y) p_{\tilde{Z}_j}(\tilde{z}_j)} \, dy d\tilde{z}_j
        \\
        &=: I(Y; \tilde{Z}_j).
    \end{align*}
    This completes the proof of \eqref{eq::IYH}.

\textbf{\textit{(ii)}} [Proof of \eqref{eq::IYZ}]:
By the definition of \textit{Nested Dropout} in \eqref{eq::mask_nested} and \eqref{eq::CatGaussian}, we have
\vspace{-1mm}
\begin{align*}
    Z=[\tilde{Z}_{1:k}, 0,\ldots,0] \,\, \text{with} \,\, 
    k \sim \mathcal{C}(p_1,\ldots, p_K), 
\end{align*}
\vspace{-1mm}
We first rewrite the above calculation equivalently as 
\begin{align} \label{eq::net_struct}
    \begin{split} 
    &T_k^1 = \tilde{Z}_k, \quad \forall \, k=1, \ldots, K 
    \\
    &T_{k}^{i+1} = 
    \begin{cases}
    T_{k}^i\, , & k \leq i 
    \\
    \varepsilon_i \, T_{k}^i \, , & k > i 
    \end{cases} \, , \\
    &\forall\, i = 1, \ldots, K-1, \quad \forall\, k=1, \ldots, K
    \\
    &\text{where} \, 
    \varepsilon_i = 
    \begin{cases}
    b \sim \textit{Bernoulli}\,(\eta_i)\, , & \varepsilon_{i-1} = 1\, ,
    \\
    0\, , & \text{o.w.}
    \end{cases} \\
    & \text{and}\,\, \varepsilon_0 = 1, \quad
    Z_k = T_k^k, \quad \forall \, k=1, \ldots, K.
    \end{split}
\end{align}
The above recursion can be solved and specified by 
\vspace{-1mm}
\begin{align*}
    Z_1 = \tilde{Z}_1, \quad
    Z_k = \Big( \prod_{j=1}^{k-1} \varepsilon_j \Big) \tilde{Z}_k, \quad \forall k=2, \ldots, K,
\end{align*}
which suggests that
\begin{align*}
    \mathrm{P}( Z=[\tilde{Z}_{1}, 0,\ldots,0] ) 
    & = \mathrm{P}(Z_1 = \tilde{Z}_1, Z_{2}=0) \\
    & = \mathrm{P}( \varepsilon_1 = 0) 
    = 1 - \eta_1,
\end{align*}
and for $k = 2, \ldots, K-1$, we have
\begin{align*}
    & \mathrm{P}( Z=[\tilde{Z}_{1:k}, 0,\ldots,0] ) 
    = \mathrm{P}(Z_k = \tilde{Z}_k, Z_{k+1}=0) \\
    & = \mathrm{P}( \varepsilon_1 = 1, \ldots, \varepsilon_{k-1} = 1, \varepsilon_{k} = 0)  
    = \Big(\prod_{j=1}^{k-1} \eta_j \Big) (1-\eta_{k}),
\end{align*}
and also have
\begin{align*}
    \mathrm{P}( Z=\tilde{Z} ) & = \mathrm{P}(Z_{K-1} = \tilde{Z}_{K-1}, Z_{K}=\tilde{Z}_K) \\
    & = \mathrm{P}( \varepsilon_1 = 1, \ldots, \varepsilon_{K-1} = 1) 
     = \prod_{j=1}^{K-1} \eta_j.
\end{align*}
Hence, the equivalence holds if
\begin{align*}
    p_1 \equiv 1 - \eta_1, \quad
    p_k \equiv \Big(\prod_{j=1}^{k-1} \eta_j \Big) (1-\eta_{k}) \\ \text{for}\,\,\, k = 2, \ldots, K-1,\,\,\, \text{and} \,\,\,
    p_K \equiv \prod_{j=1}^{K-1} \eta_j.
\end{align*}
More precisely, the equivalence holds if
\begin{align*}
    \eta_1 = 1 - p_1, \quad
    \eta_k = \frac{1 - \sum_{j=1}^k p_j}{1 - \sum_{j=1}^{k-1} p_j} \,\,\, \text{for}\,\,\, k = 2, \ldots, K-1.
\end{align*}
Given the construction of the network, we obtain a Markov chain $Y \rightarrow X \rightarrow T_k^1 \rightarrow \cdots \rightarrow T_k^K$, which simply means that $T_k^i$ is conditionally independent of $Y$ given $X$ or $T_k^{j}$ for $j < i$.
Now, we need the following data processing inequality:
\begin{lemma}[Data processing inequality]
Let three random variables form the Markov chain $A \rightarrow B \rightarrow C$, meaning that $C$ is conditionally independent of $A$ given $B$. Then,
$$
I(A; B) \geq I(A; C).
$$
The equality holds if and only if $I(A; B \,|\, C) = 0$.
\end{lemma}

Combining the above data processing inequality with both \eqref{eq::IYH} and \eqref{eq::net_struct}, we have
\begin{align*}
    I(Y; Z_{k+1}) = I(Y; T_{k+1}^{k+1}) 
    & \leq I(Y; T_{k+1}^k) \\
    & = I(Y; T_k^k) = I(Y; Z_k).
\end{align*}
Note that $I(Y; T_{k+1}^k) = I(Y; T_k^k)$ is due to the fact that
\begin{align*}
    &T_{k+1}^k = \prod_{j=1}^{k-1} \varepsilon_j \tilde{Z}_{k+1} = \begin{cases}
    \tilde{Z}_{k+1}, & \prod_{j=1}^{k-1} \varepsilon_j = 1\, ,\\
    0, & \text{o.w.}
    \end{cases} \\
    &T_{k}^k = \prod_{j=1}^{k-1} \varepsilon_j \tilde{Z}_{k} = 
    \begin{cases}
    \tilde{Z}_{k}, & \prod_{j=1}^{k-1} \varepsilon_j = 1\, ,\\
    0, & \text{o.w.}
    \end{cases}
\end{align*}
and $I(Y; \tilde{Z}_{k+1}) = I(Y; \tilde{Z}_k)$.
This completes the proof.
\end{proof}

\begin{proof} [Proof of Theorem \ref{thm::co_decomp}]
    Recall that the taught student decoder is defined by
    $q_{\text{co}}(y|x,z) = \exp[q_t(y|x) \log q(y|z)]/C_1(x,z)$
    where $C_1(x,z):=\int_{\mathcal{Y}}\exp[q_t(y|x) \log q(y|z)]\,dy$.
    For a given input $x$,
    \begin{align} \label{eq::decomp_co}
        & \mathbb{E}_{p(y|x)}[L_q^s(x,Y)]
        = \mathbb{E}_{p(y|x)} \mathbb{E}_{q(z|x)} \big[ 
        - q_t(y|x) \log q(y|z)\big]
        \nonumber\\
        & = \mathbb{E}_{p(y|x)} \mathbb{E}_{q(z|x)} 
        \big[-\log q_{\text{co}}(y|x,z) \big] - \mathbb{E}_{q(z|x)}[\log C_1(x,z)]
        \nonumber\\
        & = \mathbb{E}_{p(\varepsilon)} \mathbb{E}_{q(z|x)}
        D_{\text{KL}} \big(p(y|x,\varepsilon) \| q_{\text{co}}(y|x,z)\big) + \text{const}
    \end{align}
    where $\text{const}=H(Y|X=x,\varepsilon) - \mathbb{E}_{q(z|x)}[\log C_1(X=x,z)]$, and the last equation is according to \eqref{eq::decomp}.
    With Lemma \ref{lm::decomp} applied, the first term in \eqref{eq::decomp_co} can be decomposed as
    \begin{align*}
        & \mathbb{E}_{p(\varepsilon)} \mathbb{E}_{q(z|x)} 
        D_{\text{KL}} \big(p(y|x,\varepsilon) \| q_{\text{co}}(y|x,z)\big)
        \\
        & = D_{\text{KL}} \big( p(y|x) \| \tilde{q}_{\text{co}}(y|x)\big)
        \\
        & \quad \quad \quad + \mathbb{E}_{q(z|x)} D_{\text{KL}}
        \big(\tilde{q}_{\text{co}}(y|x) \| q_{\text{co}}(y|x,z) \big)
        + I(Y;\varepsilon)
    \end{align*}
    where we have $p(y|x)=\mathbb{E}_{p(\varepsilon)}p(y|x,\varepsilon)$, and also 
    \begin{align*}
        &\tilde{q}_{\text{co}}(y|x) \propto \exp \big[\mathbb{E}_{q(z|x)} \log q_{\text{co}}(y|x,z) \big]
        \\
        & \propto \exp \big[\mathbb{E}_{q(z|x)}q_t(y|x) \log q(y|z) \big]
        \\
        & = \exp \big[ q_t(y|x) \mathbb{E}_{q(z|x)}\log q(y|z)\big]
        \\
        & \propto \exp[ q_t(y|x)] \tilde{q}(y|x).
    \end{align*}
    Since $\tilde{q}_{\text{co}}(y|x) \propto \exp[ q_t(y|x)] \tilde{q}(y|x)$, then we have $\tilde{q}_{\text{co}}(y|x)$ follows
    $\tilde{q}_{\text{co}}(y|x)=\exp[ q_t(y|x)] \tilde{q}(y|x)/C_2(x)$ with $C_2(x):= \int_{\mathcal{Y}}\exp[ q_t(y|x)] \tilde{q}(y|x)\,dy$.
    The proof follows if we plug in the above decomposition when computing $\mathbb{E}_{p(x,y)}\big[L_q^s(X,Y)\big]$.
    Note that both $H(Y|X,\varepsilon)$ and $I(Y;\varepsilon)$ are constant with respect to the model $q(y|x)$. 
    
    \textit{(i)} For the bias term, we have
    \begin{align*} 
        & D_{\text{KL}} \big( p(y|x) \| \tilde{q}_{\text{co}}(y|x)\big)
        = \int_\mathcal{Y} p(y|x) \log \frac{p(y|x)}{\tilde{q}_{\text{co}}(y|x)} \, dy
        \nonumber \\
        & = \int_\mathcal{Y} p(y|x) \log \frac{C_2(x)\cdot p(y|x)}{\exp[ q_t(y|x)] \tilde{q}(y|x)}\, dy
        \nonumber \\
        & = \int_\mathcal{Y} p(y|x) \log \Big[ \alpha(y|x)\cdot\frac{ p(y|x)}{\tilde{q}(y|x)} \Big]\, dy.
    \end{align*}
    where $\alpha(y|x): = \mathbb{E}_{\tilde{q}(y|x)}\big[\exp[ q_t(y|x)]\big] / \exp[ q_t(y|x)]$.
    For $\alpha(y|x) \leq 1$, that is, $q_t(y|x) \geq \log\big[\mathbb{E}_{\tilde{q}(y|x)}\exp[ q_t(y|x)]\big]$, we have Co-teaching bias term satisfying 
    \begin{align*}
         D_{\text{KL}} \big( p(y|x) \| \tilde{q}_{\text{co}}(y|x)\big)
         \leq D_{\text{KL}} \big( p(y|x) \| \tilde{q}(y|x)\big).
    \end{align*}
    
    \textit{(ii)} For the variance term, we have
    \begin{align*}
        & \mathbb{E}_{q(z|x)} D_{\text{KL}}
        \big(\tilde{q}_{\text{co}}(y|x) \| q_{\text{co}}(y|x,z) \big)
        \\
        & = \mathbb{E}_{q(z|x)} \int_{\mathcal{Y}} \tilde{q}_{\text{co}}(y|x) \log \frac{\tilde{q}_{\text{co}}(y|x)}{q_{\text{co}}(y|x,z)}\, dy
        \\
        & = \mathbb{E}_{q(z|x)} \int_{\mathcal{Y}} \frac{\tilde{q}(y|x)}{\alpha(y|x)} \log \frac{\tilde{q}(y|x) /\alpha(y|x)}{\exp[ q_t(y|x) \log q(y|z)]/C_1(x,z)}\, dy
        \\
        & \geq \mathbb{E}_{q(z|x)} \int_{\mathcal{Y}} \frac{\tilde{q}(y|x)}{\alpha(y|x)} \log \Big[
        \frac{C_1(x,z)}{\alpha(y|x)}
        \frac{\tilde{q}(y|x)}{q(y|z)} \Big]\, dy
    \end{align*}
    where the last inequality is due to $\exp[ q_t(y|x) \log q(y|z)] \leq q(y|z)$ as the output of the teacher network is derived by soft-max layer with $0\leq q_t(y|x) \leq 1$. 
    Similarly, $C_1(x,z):=\int_{\mathcal{Y}}\exp[q_t(y|x) \log q(y|z)]\,dy \leq \int_{\mathcal{Y}}\exp[\log q(y|z)]\,dy = 1$.
    In these regards, if we further have $\alpha(y|x) \leq C_1(x,z)$, then
    \begin{align*}
        & \mathbb{E}_{q(z|x)} \int_{\mathcal{Y}} \frac{\tilde{q}(y|x)}{\alpha(y|x)} \log \Big[
        \frac{C_1(x,z)}{\alpha(y|x)} 
        \frac{\tilde{q}(y|x)}{q(y|z)} \Big]\, dy
        \\
        & \geq \mathbb{E}_{q(z|x)} \int_{\mathcal{Y}} \tilde{q}(y|x) \log \frac{\tilde{q}(y|x)}{q(y|z)}\, dy
        \\
        & = \mathbb{E}_{q(z|x)} D_{\text{KL}} \big(\tilde{q}(y|x) \| q(y|z)\big).
    \end{align*}
    To conclude, if $\alpha(y|x) \leq C_1(x,z)$, then 
    \begin{align*}
        \mathbb{E}_{q(z|x)} 
        & D_{\text{KL}} 
        \big(\tilde{q}_{\text{co}}(y|x) \| q_{\text{co}}(y|x,z) \big) 
        \\
        & \geq \mathbb{E}_{q(z|x)} D_{\text{KL}} \big(\tilde{q}(y|x) \| q(y|z)\big).
    \end{align*}
    This completes the proof.
\end{proof}

\end{document}